\pdfoutput=1
\documentclass[preprint,11pt,a4paper]{elsarticle}
\biboptions{numbers,sort&compress}
\usepackage{amsmath,amssymb,amsthm}
\usepackage{booktabs}
\usepackage{algorithm}
\usepackage{graphicx}
\usepackage[section]{placeins}
\usepackage{arxiv}
\usepackage{needspace}
\usepackage{xurl}

\newtheorem{theorem}{Theorem}
\newtheorem{lemma}{Lemma}
\newtheorem{corollary}{Corollary}
\theoremstyle{remark}
\newtheorem{remark}{Remark}

\newcommand{\Hk}{\mathcal{H}_k}
\newcommand{\phik}{\phi_k}
\newcommand{\x}{x}
\newcommand{\thetahat}{\widehat{\theta}}
\newcommand{\what}{\widehat{w}}
\newcommand{\ESS}{\mathrm{ESS}}

\newcommand{\appref}[1]{%
   \ref{#1}%
  }

\begin{document}

\begin{frontmatter}

\title{Density-Ratio Rescoring for Imbalanced Classification Using Raking Duals and Classifier Scores}

\author[sung]{Dongha Kim}
\author[kang]{Seunghwan Park\corref{cor}}
\ead{stat.shpark@kangwon.ac.kr}
\cortext[cor]{Corresponding author.}
\affiliation[sung]{organization={School of Mathematics, Statistics and Data Science, Sungshin Women's
  University}, addressline={2, Bomun-ro 34da-gil, Seongbuk-gu}, city={Seoul},
  postcode={02844}, country={Republic of Korea}}
\affiliation[kang]{organization={Department of Information Statistics,
  Kangwon National University}, addressline={1, Kangwondaehak-gil},
  city={Chuncheon-si, Gangwon-do}, postcode={24341},
  country={Republic of Korea}}

\begin{abstract}
Density-Ratio Rescoring (DRR) augments a classifier trained at the original
class prior with a survey-raking dual score. Raking reweights the majority
sample to match minority feature moments within a tolerance. DRR marginally
standardizes the dual and base scores and combines them with a fixed weight
of one half, using the fitted dual directly for prediction without resampling
or refitting the base classifier. Under exact population matching and a
correctly specified log-linear tilt model, the dual equals the log density
ratio up to an additive constant. A class-separation analysis characterizes
the signal strength and correlation conditions under which fusion improves
separation under common within-class covariance. On 24 tabular benchmarks,
evaluated over 30 trials and five base learners, DRR at the $D=128$
random-feature setting improves average precision over the standardized base
on every dataset, with a mean gain of $0.034$. It exceeds the shared-dual
raking-and-relabeling resampler on 22 of 24 datasets, with a mean gain of
$0.092$, and on all eight one-versus-rest tasks of a shared gene-expression
cohort. These results demonstrate the effectiveness of using raking duals as
reusable scores for improving rare-class ranking while retaining classifiers
trained at the original prior.
\end{abstract}


\end{frontmatter}

\section{Introduction}
Class imbalance is common in fraud detection, medical diagnosis, fault and
intrusion detection, and rare-disease genomics. When most observations belong to
the majority class, classifiers trained to minimize overall error may miss many
minority instances. Methods for addressing this problem operate at three levels.
\emph{Data-level} methods rebalance the training set through over- or
under-sampling, as in SMOTE \citep{smote}. \emph{Algorithm-level} methods modify
the learning procedure through cost-sensitive losses \citep{elkan} or
imbalance-aware ensembles \citep{balancedrf,rusboost}. \emph{Decision-level}
methods adjust the prediction rule, for example by moving a threshold or shifting
a score \citep{menon2013,menon2021}. We develop a decision-level method based on
survey calibration.

Survey calibration has so far entered imbalanced learning at the data level.
Raking-and-relabeling (R\&R) uses it to construct a rebalanced training sample
\citep{park2024}. It first assigns weights to majority
observations so that their weighted feature moments approximate those of the
minority. It then draws majority observations according to these weights,
relabels the selected observations as minority, and retrains a classifier near a
balanced prior. The same construction has been applied to one-class
classification~\citep{lee2023}. The calibration weights are an
exponential tilt of the empirical majority distribution, parameterized by a
fitted dual direction $\thetahat$.

R\&R computes $\thetahat$ only to construct sampling weights and then discards
it. The dual is itself a discriminant: the Bayes log-odds decompose as
$\log(\pi_1/\pi_0)+\log(dF_1/dF_0)(x)$, and under a log-linear tilt model the
population counterpart of the dual score $\langle\thetahat,\phik(x)\rangle$
equals the second term up to a constant (Lemma~\ref{lem:dr}). Density-Ratio Rescoring (DRR) uses this score
directly. It standardizes the dual score and adds it to the standardized score of
a base classifier trained on the original imbalanced data
(Figure~\ref{fig:schematic}). The base classifier is fitted once, and the dual is
estimated on the same fit subset; a Platt map fitted on held-out data at the
original class prior converts the fused score to a probability.

\begin{figure}[t]
\centering
\includegraphics[width=0.98\linewidth]{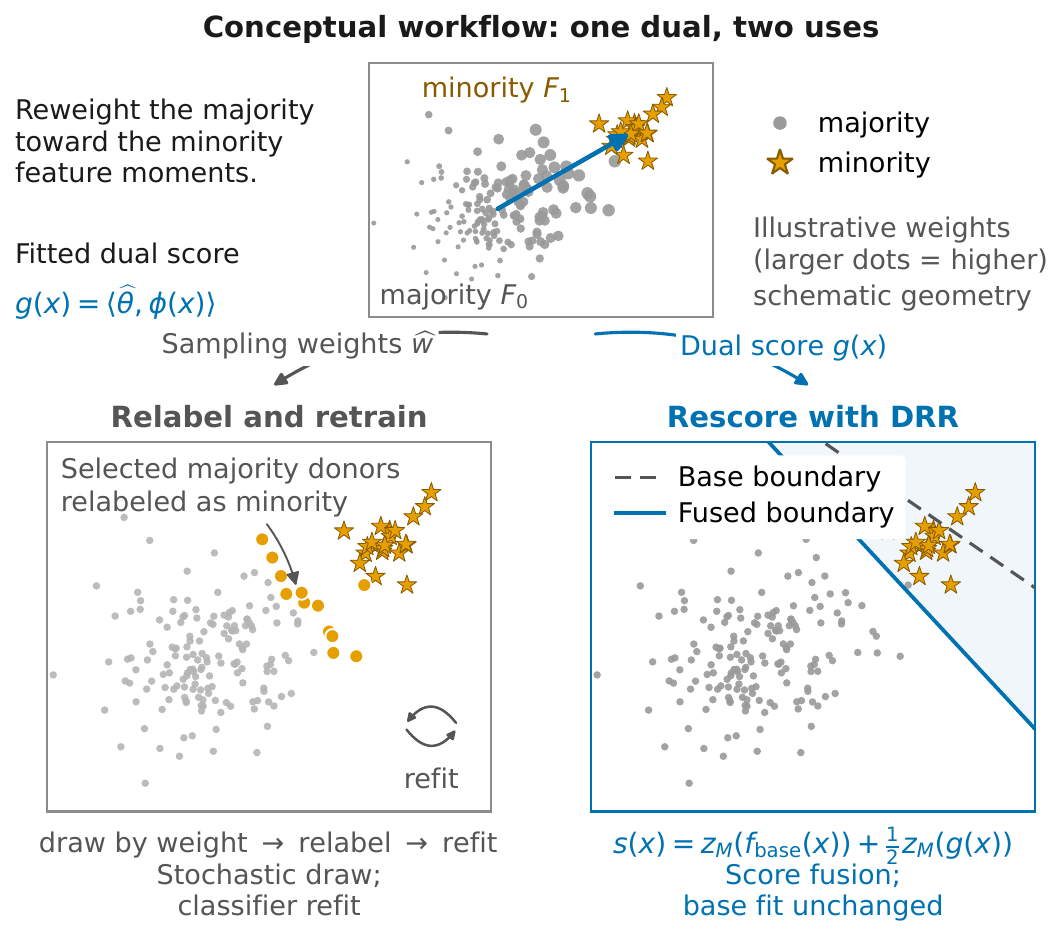}
\caption{Two uses of the fitted raking dual. Raking reweights the majority to
approximate the minority feature mean. R\&R uses the induced weights to sample
and relabel majority donors before refitting a classifier; DRR combines the dual
score with the score of an original-prior base classifier. The points, weights,
directions, and decision boundaries are schematic.}
\label{fig:schematic}
\end{figure}

The identity in Lemma~\ref{lem:dr} holds at the population level under exact
moment matching; the implemented finite-sample dual is a regularized
moment-matching score (Remark~\ref{rem:est}). The base classifier supplies a
posterior-related score learned from all labeled fit observations, so the two
constructions can have different estimation errors. Under a common
within-class covariance model, their fixed-weight combination improves class
separation when the dual supplies enough signal beyond what its correlation with
the base explains (Theorem~\ref{thm:fusion}). When marginal and within-class
scales are close, an uncorrelated dual whose standardized separation is about
one quarter of the base's can already improve separation at the deployed
weight (Remark~\ref{rem:w}).

Because DRR and R\&R start from the same fitted dual, comparing them isolates
how the dual is used: added to the score, or converted into weights for
relabeling and retraining. Relabeling shifts the training prior toward one half
and, when the weights concentrate, draws the same majority donors repeatedly;
conditional on the fitted weights, the draw adds first-stage sampling variance
(\appref{app:pps}). We also report how the DRR--R\&R gap varies with the
feasibility floor, the distance from the minority mean embedding to the majority
feature hull.

On 24 tabular datasets, DRR with the fixed weight $w=1/2$ improves average
precision over its base on every dataset at $D=128$ and exceeds the R\&R
resampler on 22; all eight one-vs-rest tasks of a shared gene-expression cohort
also favor DRR over the R\&R resampler. The dual score alone is already competitive at $D=128$, where
fusion mainly adds consistency across datasets; at $D=0$ fusion supplies most of
the gain. DRR has the highest mean average precision in the main comparator
panel (Table~\ref{tab:panel}). A single fitted dual serves all five
base-classifier families, making rescoring a practical way to use the raking
construction with existing classifiers.

\section{Related work}
A common response to class imbalance operates at the data level, rebalancing
the training set before a standard classifier is fit. SMOTE \citep{smote}
synthesizes minority points by interpolating between nearest neighbors, and a
large family of refinements has followed: Borderline-SMOTE \citep{borderline} and
ADASYN \citep{adasyn} concentrate synthesis near the decision boundary, SVM-SMOTE
\citep{svmsmote} near the support vectors, hybrids such as SMOTE-ENN
\citep{smoteenn} clean the result with edited nearest neighbors, and ROSE
\citep{rose} smooths a bootstrap. Recent variants reweight features
\citep{fwsmote} or adapt the synthetic density
\citep{zhang2020asn}. These methods generate or modify samples in feature space
and therefore rely on a meaningful neighborhood, distance, or smoothing geometry,
which can be difficult to define for categorical or high-dimensional data.
Raking-and-relabeling \citep{park2024} avoids
interpolation by reweighting the majority with survey-calibration weights and
relabeling the reweighted units, a construction also used for one-class
classification \citep{lee2023}; DRR keeps its calibration step and discards its
resampling step.

Algorithm-level methods modify the loss or the ensemble construction.
Cost-sensitive learning reweights the loss by class \citep{elkan}, and imbalance
ensembles combine undersampling with bagging or boosting, as in balanced random
forests \citep{balancedrf}, RUSBoost \citep{rusboost}, and EasyEnsemble
\citep{easyensemble}. The design of such ensembles remains an active line of work
\citep{sun2015,pang2025}. Label-distribution-aware margins instead alter
class-dependent training margins \citep{cao2019}. Decision-level methods change
neither the training data nor the fitted base classifier, but adjust its outputs
or operating rule. On calibrated scores the
balanced-accuracy-optimal threshold equals the minority prior $\pi_1$
\citep{menon2013} and the $F_1$-optimal threshold is half the optimal $F_1$
value \citep{lipton,koyejo}, while logit adjustment \citep{menon2021} shifts
logits by class-prior terms.
These post-hoc operations preserve the ranking in the binary case. DRR also
operates at the decision level, but its input-dependent dual-score term can
change the ranking. We compare it with logit adjustment in \S\ref{sec:results}.

DRR relates to probability calibration, density-ratio estimation, and score
combination. Rebalancing changes the effective training prior, so probability
estimates may require correction for the original prior. \citet{elkan} gives the
multiplicative correction, and \citet{pozzolo} analyze the posterior shift caused
by undersampling. Platt scaling \citep{platt} or isotonic regression can instead
fit a calibration map on held-out data. DRR avoids a rebalancing-induced training-prior shift:
it fits its base at the original empirical prior.

Density-ratio estimation underlies importance-weighted covariate-shift correction
\citep{kmm,kliep}. Direct estimators such as KLIEP or kernel mean matching could
also supply a second score; DRR instead reads it off the raking dual
\citep{deville}, in a feature space built from random features \citep{rahimi},
because the dual needs only the minority mean embedding, is already computed by
R\&R, and lets the rescoring-versus-relabeling comparison hold the estimator
fixed. \citet{li2023drc} construct a kernelized density-ratio classifier for
imbalanced data and combine it with threshold moving; DRR instead adds the dual
score to an original-prior base score.

DRR forms a fixed linear combination of standardized scores, the sum rule of
classifier combination, which \citet{kittler1998} show to be robust to
estimation error in the individual scores. Its analysis under a common-covariance bi-normal model
follows the classical optimal combination of two diagnostic markers
\citep{suliu}. Whereas stacking learns a combination from held-out predictions,
DRR uses a fixed weight and no meta-learner.

Studies of class imbalance also identify class \emph{overlap}, alongside the
imbalance ratio, as a determinant of minority-detection difficulty
\citep{vutti,mayabadi}. Our diagnostic is the feasibility floor
(\S\ref{sec:dual}), the distance of the minority empirical mean embedding from
the majority empirical feature hull: a sample-level analogue of overlap that also
depends on the feature map and on the sample sizes. Large-scale benchmarks provide a broader
comparison of methods. CLIMB \citep{climb}, spanning $73$ datasets and $29$
algorithms, reports that naive rebalancing often fails and that ensemble variants
can be strong. TILBench \citep{tilbench}, covering more than $40$ algorithms and
$57$ datasets, instead emphasizes regime dependence: algorithm-level methods are
generally stable, ensembles are especially competitive on small binary tasks, and
data-level methods become more competitive for multiclass or severely imbalanced
tasks. We therefore compare DRR with methods at all three levels.

\section{Method}\label{sec:method}

\subsection{Setup and the raking dual}\label{sec:dual}
We treat binary classification with a rare positive class. Write $F_1,F_0$ for the
class-conditional laws of the minority and majority, $\pi_1=1-\pi_0$ for the
minority prior, and $n_1\ll n_0$ for the two sample sizes. A base classifier
$f_{\mathrm{base}}$ is trained on the imbalanced sample and returns a score --- a
probability or a margin. Let $\phik:\mathcal{X}\to\Hk$ be the fitted finite
feature map used for raking. It always contains the linear/one-hot coordinates;
the $D=128$ variant appends 128 Gaussian random Fourier features \citep{rahimi}
when continuous variables are present, whereas $D=0$ omits that block, so that
its dual is the first-moment direction that R\&R's constraints already imply. Active blocks use
the equal-block normalization specified in \appref{app:repro}. The
minority mean embedding is $\widehat\mu_1=\tfrac1{n_1}\sum_{i:y_i=1}\phik(\x_i)$.
Raking-and-relabeling (R\&R) reweights the majority sample to match
$\widehat\mu_1$, and DRR uses the dual $\thetahat$ of that calibration. All
results are stated for this binary setup, including the eight one-vs-rest
gene-expression tasks in our evaluation.

We define the raking problem and its dual. Stack the majority embeddings as
$\Phi=[\phik(\x_i)]_{i:y_i=0}$ and let $d\in\Delta$ have strictly positive
coordinates (uniform unless noted). The calibrated weights are the
$I$-projection \citep{csiszar1975} of $d$ onto the minority-mean constraint,
\begin{equation}\label{eq:primal}
  \what=\arg\min_{w\in\Delta}\ \mathrm{KL}(w\,\|\,d)
  \quad\text{s.t.}\quad \|\Phi^\top w-\widehat\mu_1\|\le\delta ,
\end{equation}
with $\Delta$ the simplex and $\delta$ a tolerance. The constraint is feasible
only for $\delta\ge\varphi_n:=\min_{w\in\Delta}\|\Phi^\top w-\widehat\mu_1\|$, the
\emph{feasibility floor}, positive exactly when $\widehat\mu_1$ lies outside the
majority's embedding hull. The base weights are feasible once
$\delta\ge\gamma_n:=\|\Phi^\top d-\widehat\mu_1\|$, the discrepancy at the base
weights. We use $\delta_n=\max\{\varphi_n+\eta(\gamma_n-\varphi_n),\,1.05\,\varphi_n\}$
with $\eta=0.05$. The first term places the tolerance 5\% of the way from the
feasibility floor to the base discrepancy; the second imposes a 5\% buffer above
a positive floor. These fixed, dataset-relative choices keep the tolerance above
the feasibility boundary whenever $\gamma_n>0$. In computation, pairwise
Frank--Wolfe with a duality-gap stopping criterion supplies a feasible residual
upper bound for $\varphi_n$ (\appref{app:repro}).
For $\delta_n>\varphi_n$ the primal satisfies Slater's condition, strong duality
holds, and the dual below has a finite minimizer. When $\delta_n\ge\gamma_n$ the
base weights are already feasible, $\thetahat=0$, $\what=d$, and the dual
contributes nothing to the centered DRR score.

By Lagrangian duality, \eqref{eq:primal} is an exponential tilt,
\begin{equation}\label{eq:dual}
\begin{aligned}
  \what_i &\propto d_i\,e^{\langle\thetahat,\phik(\x_i)\rangle_{\Hk}},\\[4pt]
  \thetahat &=\arg\min_{\theta}\left\{
     \log\sum_{i:y_i=0} d_i\,e^{\langle\theta,\phik(\x_i)\rangle}
     -\langle\theta,\widehat\mu_1\rangle+\delta_n\|\theta\|
  \right\},
\end{aligned}
\end{equation}
the survey-calibration dual \citep{deville}. DRR uses the score defined by
$\thetahat$; Lemma~\ref{lem:dr} gives its population density-ratio
interpretation and Remark~\ref{rem:est} describes the fitted score. The
objective in \eqref{eq:dual} is convex. After an exact check for the zero
solution, Algorithm~\ref{alg:dual} minimizes it by damped Newton iteration on
the smoothed penalty $\delta_n\sqrt{\|\theta\|^2+\varepsilon_s^2}$ with
$\varepsilon_s=10^{-10}$, which changes the objective by at most
$\delta_n\varepsilon_s$. Above a few hundred feature dimensions, including all
gene-expression runs, L-BFGS replaces the Newton step; solver settings are
listed in \appref{app:repro}. Writing $K=\dim\Hk$ for the feature dimension and
$n_0$ for the number of majority points in the fit subset, a Newton step costs
$O(n_0K^2+K^3)$ for forming and factoring the weighted covariance, each L-BFGS
or pairwise Frank--Wolfe iteration $O(n_0K)$, and scoring a test point $O(K)$;
the dual is solved once per fit subset regardless of the number of base
learners.

\begin{algorithm}[t]
\caption{Computing the raking dual $\thetahat$}\label{alg:dual}
\textbf{Input:} majority embeddings $\{\phik(\x_i)\}_{i:y_i=0}$, base weights $d$,
minority mean $\widehat\mu_1$, tolerance $\delta_n$.\\
\textbf{Output:} dual direction $\thetahat$ and calibration weights $\what$.
\begin{enumerate}\setlength\itemsep{1pt}
\item Initialize $\theta\leftarrow 0$; if $\gamma_n\le\delta_n$ the base weights
      are already feasible: return $\thetahat=0$, $\what=d$.
\item \textbf{repeat}
\item \quad tilt the majority: $w_i\leftarrow d_i\,e^{\langle\theta,\phik(\x_i)\rangle}$,
      normalized so that $\sum_i w_i=1$;
\item \quad set $r_\varepsilon\leftarrow\sqrt{\|\theta\|^2+\varepsilon_s^2}$ and
      gradient $g\leftarrow\sum_i w_i\,\phik(\x_i)-\widehat\mu_1
      +\delta_n\theta/r_\varepsilon$;
\item \quad Hessian $H\leftarrow\operatorname{Cov}_w[\phik]
      +\delta_n(I/r_\varepsilon-\theta\theta^\top/r_\varepsilon^3)+\epsilon_r I$,
      the $w$-weighted embedding covariance plus the smoothed penalty curvature
      and a small ridge $\epsilon_r=10^{-10}$;
\item \quad damped Newton step
      $\theta\leftarrow\theta-\alpha_{\mathrm{bt}}H^{-1}g$, the step size
      $\alpha_{\mathrm{bt}}\in(0,1]$ chosen by backtracking on the objective;
\item \textbf{until} $\|g\|$ is below tolerance; on exit report the achieved
      discrepancy $\big\|\sum_i w_i\,\phik(\x_i)-\widehat\mu_1\big\|$ separately.
\item \textbf{return} $\thetahat\leftarrow\theta$ and $\what\leftarrow w$.
\end{enumerate}
\end{algorithm}

\begin{lemma}[Population density-ratio identity for the raking dual]\label{lem:dr}
Work at the population level: take $F_0$ itself as the reference measure, the
population minority mean $\mu_1=\mathbb{E}_{F_1}[\phik]$ as the target, and exact
matching ($\delta=0$). Assume $F_1\ll F_0$, the relevant Bochner moments and KL
divergences are finite, and the log-partition is finite and differentiable in a
neighborhood of the parameters below. Suppose the log density ratio is
log-linear in the features,
$\log\frac{dF_1}{dF_0}(\x)=\langle\theta^\star,\phik(\x)\rangle+c$ for some
$\theta^\star\in\Hk$, $c\in\mathbb{R}$, and that the exponential family is
minimal: if $\langle h,\phik(X)\rangle$ is constant $F_0$-almost surely, then
$h=0$. Then $F_1$ itself is the $I$-projection of
$F_0$ onto $\{G:\mathbb{E}_G[\phik]=\mu_1\}$, an exponential tilt with dual
$\theta^\star$, so that
\begin{equation}\label{eq:dr}
  \langle\theta^\star,\phik(\x)\rangle_{\Hk}
  \;=\; \log\frac{dF_1}{dF_0}(\x) + \mathrm{const}.
\end{equation}
Without the log-linear assumption, provided an interior member of the family
$\{dG_\theta\propto \exp\langle\theta,\phik\rangle\,dF_0\}$ matches the moment
$\mathbb{E}_{G_\theta}[\phik]=\mu_1$, that moment-matching tilt is the member
closest to $F_1$ in the divergence $\mathrm{KL}(F_1\|\cdot)$.
\end{lemma}

\begin{proof}
If $dF_1=\exp(\langle\theta^\star,\phik\rangle+c)\,dF_0$ then $F_1$ belongs to
the tilt family and satisfies $\mathbb{E}_{F_1}[\phik]=\mu_1$. For any $G$ with
$\mathbb{E}_G[\phik]=\mu_1$, expanding the log-ratio gives the Pythagorean
identity of $I$-projection \citep{csiszar1975}
$\mathrm{KL}(G\|F_0)=\mathrm{KL}(G\|F_1)+\langle\theta^\star,\mu_1\rangle+c
=\mathrm{KL}(G\|F_1)+\mathrm{KL}(F_1\|F_0)$,
so $F_1$ is the unique $I$-projection; minimality makes its dual parameter
$\theta^\star$ identifiable, which is
\eqref{eq:dr}. Without the log-linear assumption,
$\mathrm{KL}(F_1\|G_\theta)=\mathrm{KL}(F_1\|F_0)
-[\langle\theta,\mu_1\rangle-\Lambda(\theta)]$ with $\Lambda$ the log-partition
of the tilt, so minimizing it is the concave problem
$\max_\theta\{\langle\theta,\mu_1\rangle-\Lambda(\theta)\}$, whose first-order
condition is the moment match $\mathbb{E}_{G_\theta}[\phik]=\mu_1$; the assumed
moment-matching member attains it. If $\mu_1$ lies outside the closed convex
hull of the support of $\phik(X)$ under $F_0$, the moment constraint is
infeasible. On its boundary the constraint may
be feasible while no finite natural parameter attains the limiting tilt; these
are distinct from the regular interior case and motivate the finite-sample
tolerance of \S\ref{sec:dual}.
\end{proof}

With a full one-hot coding, indicators within a categorical variable sum to a
constant, so parameter minimality is understood after dropping one reference
level or, equivalently, on the quotient by constant-valued directions. This does
not change the normalized tilt weights, and score standardization removes the
corresponding additive constant.

\begin{remark}[From the population identity to the deployed dual]\label{rem:est}
The deployed $\thetahat$ of \eqref{eq:dual} uses the empirical majority law,
the empirical minority mean $\widehat\mu_1$, and a positive tolerance
$\delta_n$. Lemma~\ref{lem:dr} describes exact population matching under the
stated model. With fixed $\eta=0.05$, the tolerance need not tend to zero as
sample sizes grow, so the lemma does not establish consistency of the deployed
dual for $\theta^\star$. We make no sampling-rate claim for this fitted-map,
positive-tolerance estimator. Random features approximate the chosen kernel
\citep{rahimi}; a rate for approximating a particular log-linear function would
require additional assumptions on that function. The lemma is stated in the
finite map $\phik$ itself. When $\widehat\mu_1$ lies outside the majority embedding hull
($\varphi_n>0$), exact matching is infeasible. A nonzero solution of the
exact-norm dual points along the residual from the fitted weighted majority
mean to $\widehat\mu_1$. This geometric interpretation does not establish a
density-ratio identity; a zero dual has no normalized direction. For a
characteristic feature map, equality of $\widehat\mu_1$ to a weighted majority
embedding would require equality of the corresponding empirical measures, so two
independent samples from the same continuous law can still have $\varphi_n>0$;
the floor is a sample- and representation-dependent quantity.
\end{remark}

\noindent In the first-moment (linear) case the population dual coincides with
Owen's infinitely-imbalanced logistic-regression coefficient \citep{owen}; for a
Gaussian majority law that coefficient is the mean-difference discriminant
$\Sigma^{-1}(\mu_1-\mu_0)$, which is also the log-density-ratio direction when
both class-conditional laws are Gaussian with common covariance $\Sigma$.

The original R\&R method uses $\thetahat$ for relabeling. It converts
$\thetahat$ into the weights $\what$, draws $m$ majority units by multinomial
(with-replacement) probability-proportional-to-size sampling, relabels them
minority, and \emph{retrains} a classifier on the balanced set. R\&R prescribes
capping the draw at $m\le\kappa_{\ESS}\,\ESS(\what)$, where
$\ESS(\what)=1/\sum_i\what_i^2$ is the Kish effective sample size of the
normalized weights and $\kappa_{\ESS}$ is a constant of order one. Throughout,
the \emph{R\&R resampler} denotes the comparator built from the same fitted
dual, learner, and split as DRR; it draws $m=\lfloor n_{\mathrm{fit}}/4\rfloor$
and records, but does not enforce, the cap. A separate $D=0$ diagnostic with
SVM and random-forest bases finds lower AP after capping, under its own
draw-size convention (\appref{app:repro}).
Concentrated weights lower $\ESS(\what)$, and an uncapped draw then reuses
donors. Because the retrained classifier is fit near an effective prior of
$\tfrac12$, its probabilities require prior correction or recalibration at the
original prior.

Relabeling thus adds a sampling stage that rescoring does not have. Conditional
on the data and the fitted weights, the relabeled donors form a with-replacement
PPS draw whose mean embedding $\widetilde\mu_1$ has conditional mean
$\Phi^\top\what$ and conditional covariance $\Sigma_{\what}/m$, the
$\what$-weighted majority covariance divided by the draw size; under exact
matching the law of total variance gives
$\operatorname{Var}(\widetilde\mu_1)=\Sigma_1/n_1+\mathbb{E}(\Sigma_{\what}/m)$,
the plug-in variance plus a positive-semidefinite term (\appref{app:pps}). This
first-stage identity does not determine the accuracy of the retrained
classifier; that comparison is empirical (\S\ref{sec:results}).

\subsection{The DRR score}\label{sec:proposal}
DRR uses $\thetahat$ as a second term in the decision score, not as a
resampling weight. With a base classifier $f_{\mathrm{base}}$ trained on the
\emph{original imbalanced} data, the DRR score is
\begin{equation}\label{eq:drr}
  s(\x) \;=\; z_M\big(f_{\mathrm{base}}(\x)\big)
            \;+\; w\cdot z_M\big(\langle\thetahat,\phik(\x)\rangle_{\Hk}\big),
  \qquad w=\tfrac12,
\end{equation}
where $z_M(u)=(u-\widehat m_u)/\widehat\sigma_u$ uses the marginal mean and
standard deviation fitted on a held-out threshold subset (\S\ref{sec:proto}).
If the threshold-subset scale satisfies $\widehat\sigma_u\le10^{-12}$, the
implementation sets the divisor to one to avoid division by a near-zero scale.
An identically constant score contributes zero after centering. DRR adds the
standardized dual score without resampling or refitting the base classifier.
The fixed weight $w=1/2$ gives the dual half the base's coefficient in marginal
standard-deviation units; it was fixed before the experiments and is not tuned
(sensitivity analysis in \S\ref{sec:results}).
 Figure~\ref{fig:illust} illustrates
score fusion using a simple mean-contrast direction.
Algorithm~\ref{alg:drr} specifies the full procedure. The standardization here
is by the \emph{marginal} variance; Remark~\ref{rem:w} relates it to the
within-class standardization used in the analysis.

\begin{figure}[!htbp]
\centering
\includegraphics[width=0.96\linewidth]{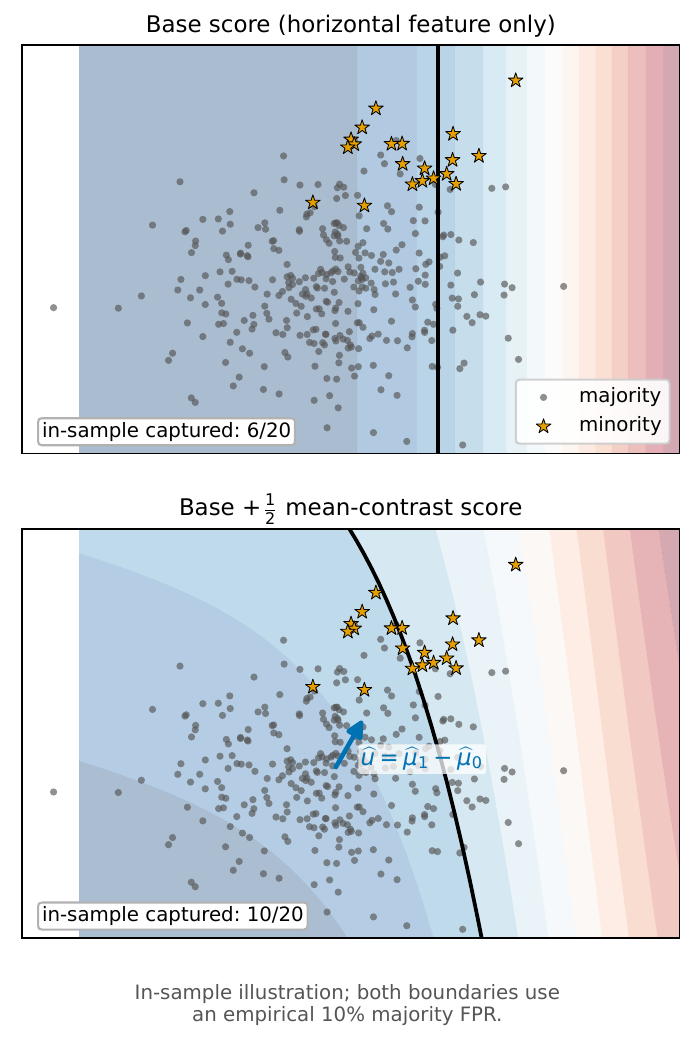}
\caption{Conceptual example of mean-contrast score fusion. The base classifier
uses only the horizontal coordinate. Adding half the standardized score
$\widehat u^\top x$, with $\widehat u=\widehat\mu_1-\widehat\mu_0$, tilts the boundary toward the
vertical minority signal. At an empirical 10\% false-positive rate, the base
captures 6 of 20 minority points and the fused score captures 10. Counts are
in-sample, and the direction is a sample mean contrast used in place of a fitted
raking dual.}
\label{fig:illust}
\end{figure}

\begin{algorithm}[t]
\caption{Density-Ratio Rescoring (DRR)}\label{alg:drr}
\begin{enumerate}\setlength\itemsep{1pt}
\item Split the training half into a fit and a threshold subset (\S\ref{sec:proto}).
\item Fit the base classifier $f_{\mathrm{base}}$ on the fit subset --- at the
      original empirical prior, with nothing rebalanced.
\item On the fit subset solve the raking dual $\thetahat$ of \eqref{eq:dual} at
      tolerance $\delta_n$ (\S\ref{sec:dual}); the dual score is
      $g(\x)=\langle\thetahat,\phik(\x)\rangle$.
\item On the threshold subset fit the marginal standardizations $z_M(\cdot)$ of
      $f_{\mathrm{base}}$ and $g$, and form the DRR score
      $s(\x)=z_M(f_{\mathrm{base}}(\x))+\tfrac12\,z_M(g(\x))$ \eqref{eq:drr}.
\item Fit a Platt map $p_{\mathrm{DRR}}(s)$ and tune the decision threshold on the
      threshold subset.
\item On the test half report $s$ for ranking and $p_{\mathrm{DRR}}(s)$ for
      probabilities and thresholded predictions.
\end{enumerate}
\end{algorithm}

\subsection{Analysis of score fusion}
The two scores provide related discriminative information. The Bayes
log-odds decompose as
$\log\frac{P(y=1\mid\x)}{P(y=0\mid\x)}
   =\log\frac{\pi_1}{\pi_0}+\log\frac{dF_1}{dF_0}(\x)$;
the base classifier supplies a posterior-related probability or margin score
fitted to all labeled points, whereas the dual supplies a moment-matching score
whose density-ratio interpretation follows from Lemma~\ref{lem:dr} and whose
only minority input is $\widehat\mu_1$. Because the two scores use different model classes and
different reductions of the data, their within-class errors need not be fully
correlated. Theorem~\ref{thm:fusion} makes the consequence precise: it expresses
the change in class separation in terms of the two individual signals and their
within-class correlation $\rho$. Throughout, $s_1$ and $s_2$ are treated as fixed functions --- the
analysis is conditional on the training data that produced $f_{\mathrm{base}}$
and $\thetahat$ --- and all moments are over the test law. The algebra is the
classical linear combination of two diagnostic markers \citep{suliu};
Remark~\ref{rem:w} specializes it to the fixed DRR weight.

\begin{theorem}[Fusion improvement]\label{thm:fusion}
Let $s_1=f_{\mathrm{base}}(\x)$ and $s_2=\langle\thetahat,\phik(\x)\rangle$ have
finite within-class second moments. Assume, for $j=1,2$,
$\operatorname{Var}(s_j\mid Y=0)=\operatorname{Var}(s_j\mid Y=1)=\sigma_j^2>0$
and
$\operatorname{Cov}(s_1,s_2\mid Y=0)=\operatorname{Cov}(s_1,s_2\mid Y=1)
=\rho\sigma_1\sigma_2$, where $-1<\rho\le1$. Assume each score, in the sign in
which it is deployed, points toward the minority class on the test law,
$d_j=\{\mathbb{E}(s_j\mid Y=1)-\mathbb{E}(s_j\mid Y=0)\}/\sigma_j\ge0$.
This orientation on the test law is an assumption for both fitted scores.
Call $(d_1,d_2)\ne(0,0)$ the nondegenerate case.
For $z_W(s_j)=(s_j-a_j)/\sigma_j$, with each $a_j$ common to the two classes,
the score $s=z_W(s_1)+wz_W(s_2)$, $w\ge0$, has deflection
\begin{equation}\label{eq:dprime}
  d'(w)\;=\;\frac{d_1+w\,d_2}{\sqrt{1+w^2+2w\rho}},
\end{equation}
and hence
\begin{enumerate}
\item[\textup{(a)}] for every $w>0$: $d'(w)>d_1$ iff
      $w(d_1^2-d_2^2)<2d_1(d_2-\rho d_1)$; when $d_1=d_2>0$ this holds for every
      $w>0$ exactly when $\rho<1$;
\item[\textup{(b)}] let $-1<\rho<1$. If $(d_1,d_2)=(0,0)$, then $d'(w)=0$ for
      every $w\ge0$. Otherwise the maximization over $[0,\infty)$ has three
      cases: if $d_2\le\rho d_1$, the unique maximizer is $w_+^\star=0$; if
      $d_2>\rho d_1$ and $d_1>\rho d_2$, the unique maximizer is
      \[
        w_+^\star=\frac{d_2-\rho d_1}{d_1-\rho d_2}>0,\qquad
        d'(w_+^\star)^2=\frac{d_1^2-2\rho d_1d_2+d_2^2}{1-\rho^2};
      \]
      if $d_1\le\rho d_2$, then $d'$ is strictly increasing on
      $[0,\infty)$ and has no finite maximizer, with
      $\sup_{w\ge0}d'(w)=d_2$ as $w\to\infty$;
\item[\textup{(c)}] if the second score has zero separation and is uncorrelated
      with the first ($d_2=0$, $\rho=0$), then
      $d'(w)=d_1/\sqrt{1+w^2}$, strictly below $d_1$ for every $w>0$ when
      $d_1>0$.
\end{enumerate}
Parts \textup{(a)--(c)} are second-moment identities and need no distributional
assumption. If in addition the scores are bi-normal (jointly Gaussian within
each class), then $\mathrm{AUC}=\Phi_{\mathrm N}(d'/\sqrt2)$, with
$\Phi_{\mathrm N}$ the standard normal distribution function, is increasing in
$d'$, so \textup{(a)--(c)} transfer to AUC.
\end{theorem}

\begin{proof}
Centering is immaterial because each $a_j$ is common to both classes; dividing by
the common within-class standard deviation gives unit conditional variances and
mean separations $d_j$. The mean separation of $s=z_W(s_1)+w z_W(s_2)$ is
$d_1+w d_2$ and its within-class variance is
$\operatorname{Var}(z_W(s_1)+w z_W(s_2)\mid y)=1+w^2+2w\rho$, giving
\eqref{eq:dprime}.
For (a), both $d'(w)$ and $d_1$ are nonnegative under the orientation, so
$d'(w)>d_1$ iff $(d_1+wd_2)^2>d_1^2(1+w^2+2w\rho)$; cancelling $d_1^2$
and dividing by $w>0$ leaves $2d_1(d_2-\rho d_1)>w(d_1^2-d_2^2)$, and at
$d_1=d_2>0$ the right side is $0$ while the left is $2d_1^2(1-\rho)$. For (b),
differentiating \eqref{eq:dprime} gives
\[
 \frac{\partial d'(w)}{\partial w}=
 \frac{(d_2-\rho d_1)-w(d_1-\rho d_2)}
 {(1+w^2+2\rho w)^{3/2}}.
\]
Write $A=d_2-\rho d_1$ and $B=d_1-\rho d_2$. If $A\le0$, then $d_2\le\rho d_1$
with $(d_1,d_2)\ne(0,0)$ gives $d_1>0$, hence $\rho\ge0$ and
$B=d_1-\rho d_2\ge d_1(1-\rho^2)>0$, so the
derivative is nonpositive at zero and strictly negative for $w>0$. If $A>0$ and
$B>0$, it changes sign once at $A/B$. If $B\le0$, then $d_1\le\rho d_2$ with
$(d_1,d_2)\ne(0,0)$ gives $d_2>0$, hence $\rho\ge0$ and $A\ge d_2(1-\rho^2)>0$,
so the derivative is strictly positive for every $w\ge0$. Substitution at $A/B$ gives the displayed value. Part (c) is
\eqref{eq:dprime} at $d_2=\rho=0$. None of this uses Gaussianity. The AUC
identity is the standard equal-variance bi-normal formula, invoked only for the
last claim.
\end{proof}

Average precision and $F_1$ are important in imbalanced classification, but a gain
in $d'$ or AUC need not improve precision at a given recall. Under the
equal-variance bi-normal model, the following corollary links separation to ROC
dominance and to population AUPRC, the continuous-score counterpart of the
empirical average precision of \S\ref{sec:metrics}.

\begin{corollary}[ROC dominance, precision, and AUPRC]\label{cor:ap}
Under the equal-variance bi-normal model an ROC depends on the score only
through its deflection,
$\mathrm{TPR}=\Phi_{\mathrm N}\!\big(\Phi_{\mathrm N}^{-1}(\mathrm{FPR})+d'\big)$.
Hence whenever $d'(w)>d'(0)$ the DRR score's ROC dominates the base's at every
interior false-positive rate (the curves coincide at the endpoints). At a fixed
minority prior $0<\pi_1<1$, precision
$=\pi_1\mathrm{TPR}/\big(\pi_1\mathrm{TPR}+(1-\pi_1)\mathrm{FPR}\big)$ is
increasing in $\mathrm{TPR}$ and decreasing in $\mathrm{FPR}$, so raising the ROC
raises precision at every interior recall. Integrating, both population AUPRC
and AUC strictly increase.
\end{corollary}

\begin{proof}
The displayed ROC is the equal-variance bi-normal form. Fix $\mathrm{FPR}$; a
larger $d'$ raises $\mathrm{TPR}$, i.e.\ the ROC moves up pointwise. Precision at
recall $r$ is precision at $\mathrm{TPR}=r$, evaluated at the $\mathrm{FPR}$ the
dominant ROC attains there, which is no larger; monotonicity of precision then
gives the pointwise precision gain, and integrating precision over recall gives
population AUPRC, while integrating the ROC over FPR gives AUC.
\end{proof}

\begin{remark}[Scope of the bi-normal assumption]\label{rem:df}
The moment-based separation result also yields a distribution-free error bound.
By Cantelli's inequality a threshold midway between the class
means misclassifies each class with probability at most $1/(1+(d'/2)^2)$ for
\emph{any} score law with the common within-class variance, so raising $d'$
(Theorem~\ref{thm:fusion}) tightens this bound. Turning a
larger $d'$ into a larger AUC and AUPRC is what requires the bi-normal model
(Corollary~\ref{cor:ap}): outside it, a rise in $d'$ can in principle coexist
with a fall in AUC, since the deflection sees only two moments.
\end{remark}

\begin{remark}[Fixed fusion weight $w=\tfrac12$]\label{rem:w}
The optimal weight in Theorem~\ref{thm:fusion} depends on $d_1,d_2,\rho$; DRR
fixes $w=\tfrac12$. In the within-class standardization of
Theorem~\ref{thm:fusion}(a), for $d_1>0$ the fixed weight beats the base
($d'(\tfrac12)>d_1$) exactly when
\[
  \frac{d_2}{d_1}\;>\;\sqrt{5+4\rho}-2,
\]
about $0.24$ at $\rho=0$ (from Theorem~\ref{thm:fusion}(a) at $w=\tfrac12$;
\appref{app:theorycurve} plots $d'(w)$). The deployed score standardizes by the
marginal, not the within-class, scale. For a score with deflection $d_j$
the two scales differ by the factor $\sqrt{1+\pi_0\pi_1d_j^2}$, so the deployed
weight corresponds to the effective within-class weight
$w_{\mathrm{eff}}=\tfrac12\sqrt{(1+\pi_0\pi_1 d_1^2)/(1+\pi_0\pi_1 d_2^2)}$,
which tends to $\tfrac12$ as $\pi_1\to0$; the $0.24$ threshold is that limit,
and the exact marginal condition is given in \appref{app:theorycurve}. In
practice the separating direction of Remark~\ref{rem:est} must therefore carry
signal beyond what its correlation with the base explains ($d_2>\rho d_1$);
whether it does is an empirical question. \S\ref{sec:results} reports the
weight sensitivity, and \appref{app:sensitivity} estimates $(d_1,d_2,\rho)$
per dataset.
\end{remark}

\section{Experimental setup}\label{sec:proto}

\subsection{Evaluated methods}
All methods use the same fit subset and are evaluated once on the test half.
\begin{center}
\renewcommand{\baselinestretch}{1}\normalsize
\begin{tabular}{@{}>{\raggedright\arraybackslash}p{0.27\linewidth}>{\raggedright\arraybackslash}p{0.63\linewidth}@{}}
\toprule
Method & Description\\
\midrule
base (\texttt{none}) & classifier on the imbalanced fit set; threshold tuned\\
cost-sensitive & class-weighted (\texttt{balanced}) base classifier \citep{elkan}\\
logit adjustment & prior-shifted decision score \citep{menon2021}\\
SMOTE family & \texttt{SMOTE}, \texttt{BSMOTE}, \texttt{ADASYN},\\
 & \texttt{ROS}, \texttt{ROSE}, \texttt{RUS}, \texttt{SMOTE-ENN},\\
 & \texttt{SVM-SMOTE}, \texttt{FW-SMOTE}\\
ensembles & balanced RF \citep{balancedrf}; RUSBoost \citep{rusboost}\\
boosting & histogram gradient boosting (HistGB) \citep{friedman2001}\\
R\&R resampler & R\&R-resample$(D)$: relabeled balanced sample $\to$ classifier, $D\in\{0,128\}$\\
dual score alone & $z_M(g)$ with its own Platt map, $D\in\{0,128\}$; independent of the base learner\\
same-map controls & logistic score on $\phik$, alone and fused with the base; raw-ranking evaluation\\
\textbf{DRR} & score \eqref{eq:drr} with direction \texttt{DRR($D{=}0$)} or \texttt{DRR($D{=}128$)}\\
\bottomrule
\end{tabular}
\end{center}

We compare DRR with data-level, algorithm-level, and decision-level methods.
At the data level we run nine resamplers (the SMOTE family, random over- and
undersampling, and ROSE) and report
three representatives in the panel tables --- the classical SMOTE, the cleaning
hybrid SMOTE-ENN, and a recent feature-weighted refinement, FW-SMOTE
\citep{fwsmote}; the remaining six are contained in the released raw results.
FW-SMOTE follows the published rank-based feature weighting and sampling
algorithm at a fixed basic-RIM setting, with at most three synthetic observations per
minority observation (details in \appref{app:repro}). The other pure over- and
undersampling arms target a training prior of $\tfrac12$; a subsequent cleaning
step such as ENN can shift the realized ratio. At the
algorithm level we include cost-sensitive learning (which reweights the loss
instead of the sample), the imbalance ensembles, and, as a modern nonlinear
learner, an unweighted histogram gradient-boosting model. At the decision level
we include threshold tuning and logit adjustment.

\subsection{Protocol}
For the main benchmark panel, we evaluate each method under the same three
cut policies: the fixed $t=\tfrac12$; the cut that maximizes balanced accuracy on
a held-out threshold subset; and the cut that maximizes $F_1$ on it. The
threshold-free metrics --- AP, ROC-AUC and Brier --- do not depend on the cut. For
the threshold-dependent metrics we report all three policies, with full tables
and a summary figure in \appref{app:cuts}; where a single number is
wanted, balanced accuracy, MCC and G-mean are read at the balanced-accuracy cut
and $F_1$ at the $F_1$ cut. Logit adjustment additionally carries its native cut, zero on the shifted
logit; being rank-preserving, its AP and AUC equal those of \texttt{none} up to
numerical ties introduced by probability clipping (\appref{app:repro}). The full
protocol is as follows.
\begin{enumerate}
\item Outer split: stratified shuffle, test size $\tfrac12$, $30$ trials, fixed
      seed.
\item Inner $70/30$ stratified fit\,/\,threshold split of the training half. All
      models are fit on the fit subset; thresholds and the marginal standardizations
      $z_M(\cdot)$ in \eqref{eq:drr} are fit on the threshold subset; both are
      applied to the test half. For the 1-Platt Brier column each pipeline's
      score is additionally Platt-scaled on the threshold subset; DRR carries
      this map inside its pipeline (Algorithm~\ref{alg:drr}).
\item Direction: $\thetahat$ solves the dual \eqref{eq:dual} at tolerance
      $\delta_n$ (\S\ref{sec:dual}); the test score is $\phik(\x)^\top\thetahat$.
\item Classifiers: SVM (RBF), random forest, naive Bayes, kernel logistic
      regression, with an MLP added on the tabular benchmarks; every method is
      run under each applicable learner. HistGB is a standalone tabular learner.
\end{enumerate}

\subsection{Metrics and prediction}\label{sec:metrics}
We report seven metrics. Two are threshold-free ranking metrics; for DRR they are
computed from the raw fused score $s$ of Algorithm~\ref{alg:drr}, not from the
Platt probability. Average precision is the recall-increment-weighted summary
\begin{equation}
  \mathrm{AP}=\textstyle\sum_k (\mathrm{R}_k-\mathrm{R}_{k-1})\,\mathrm{Prec}_k ,
\end{equation}
over the stepwise precision--recall curve, with recall $\mathrm{R}_k$ and
precision $\mathrm{Prec}_k$ at the $k$th cut; it is the primary metric because it
reflects precision for the rare class at the observed prevalence. ROC-AUC is
invariant to prevalence for fixed class-conditional score laws and complements
AP. Four standard operating-point metrics are evaluated at a single cut: balanced
accuracy (the mean of the two class recalls), $F_1$ (the harmonic mean of
minority precision and recall), the Matthews correlation coefficient (MCC, the
correlation between predicted and true labels), and G-mean (the geometric mean
of the two class recalls). The Brier score, the mean squared error of the
predicted probabilities against the labels, measures probability accuracy,
combining calibration and refinement; for DRR it uses the Platt probability.

We run three analyses: per-method means over datasets and trials; paired
multi-dataset significance comparisons; and a rank-correlation analysis of the
per-task margin $\Delta=\text{DRR}-\text{R\&R resampler}$ with the feasibility
floor $\varphi_n$ and with the normalized effective sample size $\ESS(\what)/n_0$,
both computable before any classifier is trained.

We test significance following the standard framework of \citet{demsar2006} for
comparing classifiers over multiple datasets, which avoids the distributional
assumptions of a paired $t$-test and treats each dataset as one paired
observation. For each discrimination metric we run the Friedman omnibus test
across the eleven methods of the control-versus-rest family --- DRR$(D{=}128)$
and the ten baselines listed in \appref{app:significance} --- over the $22$ benchmarks
on which all eleven have complete results.
Two post-hoc views are reported. For an all-pairs ranking
view we use the Friedman average ranks with the Nemenyi critical difference,
shown as a critical-difference diagram in \S\ref{sec:results}. For the
control-versus-rest view, in which DRR is the single method of interest, we use
the Wilcoxon signed-rank test of DRR against each competitor, paired over the
datasets, and correct the resulting family of $p$-values within each metric by
the Holm step-down procedure to hold the family-wise error rate at $\alpha=0.05$;
these corrected $p$-values are collected in \appref{app:significance}. Holm
correction is applied within each named family --- the ten baselines per metric
here, and in \S\ref{sec:results} the two resolutions of the primary contrast,
the four dual-alone contrasts, the six weight contrasts, and the separate
two-resolution same-feature LR-control family --- not across families.

\section{Results}\label{sec:results}
We run the protocol on $24$ tabular datasets (imbalance ratio $10$ to $130$ on
the full sample; see \appref{app:repro})
and eight OpenML/GEMLeR one-versus-rest tumor tasks defined on a shared
gene-expression cohort ($n=1{,}545$, $p=10{,}935$); the latter are distinct
prediction tasks, not independent cohorts (\appref{app:repro}). Each task
uses $30$ trials. We assess the dual contribution, the shared-dual R\&R
resampler, and the broader comparator panel.

\subsection{Contribution of the dual score}
The $w=0$ reference in this ablation is the marginally standardized base
score. Floating-point ties make its AP slightly different from the native
base probabilities in Table~\ref{tab:panel}; the stored ranking scores are used
throughout. Adding the dual score improves mean AP at both feature resolutions. At $w=1/2$
it raises mean tabular AP from $0.4216$ to $0.4500$ at $D=0$ (paired gain
$0.0285$, dataset-bootstrap 95\% CI $[0.0188,0.0397]$, 23 wins/1 loss) and to
$0.4554$ at $D=128$ (gain $0.0339$, $[0.0232,0.0458]$, 24/0); both have
Holm-adjusted $p=3.6\times10^{-5}$ within the two-resolution primary family.
The dual score alone reaches $0.4116$ at $D=0$ and $0.4500$ at $D=128$. At
$D=128$ it already exceeds the base by $0.028$ (17 wins/7 losses, Holm
$p=0.037$) and every baseline mean in Table~\ref{tab:panel}; fusion then adds
little to the mean ($+0.0055$, Hodges--Lehmann $0.0012$, 95\% CI
$[-0.0075,0.0202]$, 12/12, Holm $p=1.0$). At $D=0$ the dual score alone is no
better than the base ($-0.010$, 13/11), and fusion adds $0.0385$
(Hodges--Lehmann $0.0243$, CI $[0.0112,0.0691]$, 17/7, Holm $p=0.146$ for the
signed-rank test; the bootstrap interval is for the mean difference). The dual
score is thus the main source of the mean gain at $D=128$. Fusion extends
the improvement over the base to all 24
datasets, whereas the dual score alone does so on 17, and at $D=0$, where the dual
score alone is weak, it supplies most of the gain. Across the grid
$w\in\{0,0.1,0.25,0.5,1\}$ mean AP increases with $w$ (Figure~\ref{fig:weight});
relative to $w=1/2$, $w=1$ gains $0.0016$ at $D=0$ (Holm $p=0.059$) and $0.0051$
at $D=128$ (95\% CI $[0.0030,0.0072]$, 20/4, Holm $p=0.00114$). The direction
of this ordering agrees with Theorem~\ref{thm:fusion}(b): from the per-dataset
estimates of $(d_1,d_2,\rho)$ in \appref{app:sensitivity}, the implied optimum
$w_+^\star=(d_2-\rho d_1)/(d_1-\rho d_2)$ at $D=128$ has median $0.73$ over
datasets and exceeds the deployed $\tfrac12$ on 16 of 24, although it reaches
one --- the case $d_2\ge d_1$ --- on only six. Because $w=1$ is the grid endpoint, the grid does not
locate an optimum; we keep $w=1/2$ as the reported configuration and discuss a
wider cross-fitted grid in \S\ref{sec:discussion}. A supplementary sweep over
$D$, $\eta$, and the random-feature draw, together with per-dataset estimates
of $(d_1,d_2,\rho)$ and the fusion deflection of Theorem~\ref{thm:fusion}, is
reported in \appref{app:sensitivity}: the gain over the base is stable for
$D\ge64$ and $\eta\le0.2$, and the implied change in deflection tracks the
observed AP gain (Spearman $+0.73$ over dataset--learner pairs).

\begin{figure}[t]
\centering
\includegraphics[width=0.98\linewidth]{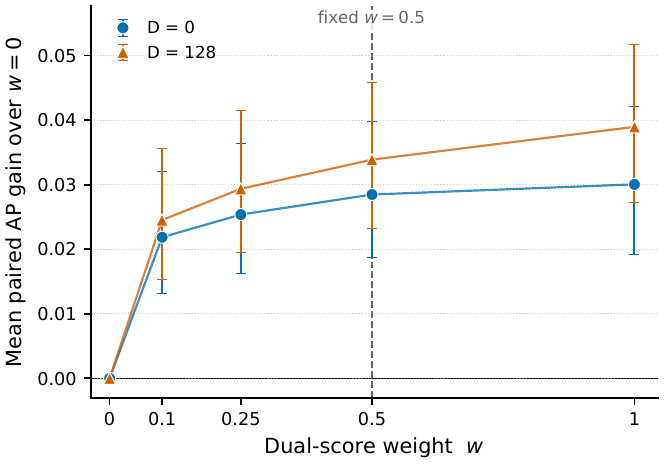}
\caption{Fixed-weight sensitivity on the 24 tabular datasets. Each point is the
mean paired raw-score AP gain, $\mathrm{AP}_d(w)-\mathrm{AP}_d(0)$, averaged over
datasets $d$. Bars are pointwise 95\% percentile intervals from 9,999 paired
dataset-bootstrap resamples. Lines connect the evaluated grid points only;
the dashed vertical line marks the fixed $w=1/2$.}
\label{fig:weight}
\end{figure}

A paired control fits logistic regression on the same $\phik$ and evaluates
its standalone and fused scores (\appref{app:samephi}). DRR has higher mean
AP than the fixed $C=1$ LR fusion, while weaker LR regularization can close
or reverse this ordering; the full protocol and results are given there.

\subsection{Comparison with the R\&R resampler}
DRR and the R\&R resampler use the same direction resolution, learner family,
and split; they differ in whether the direction is added to the score or
converted into a stochastic relabeling-and-retraining pipeline. On the 24 tabular
datasets DRR wins on AP in 22 at both resolutions, with mean gains $0.090$ at
$D=0$ and $0.092$ at $D=128$; \texttt{abalone} and \texttt{abalone\_19} are
different row subsets and labelings of the same UCI abalone table
(\appref{app:repro}) and are treated as separate paired observations. Against the R\&R resampler, DRR also wins all eight gene-expression tasks, with mean gains
$0.147$ and $0.138$. Because those tasks share one cohort we do not pool them
with the tabular datasets for inference; descriptively, DRR wins 30 of 32 tasks
overall with a mean gain near $0.10$. Figure~\ref{fig:dom} shows the task-level
comparison.

\begin{figure}[t]
\centering
\includegraphics[width=0.76\linewidth]{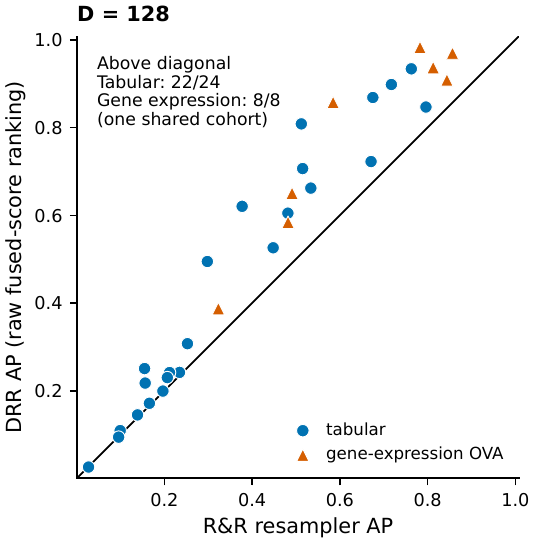}
\caption{Raw-score average precision of DRR versus the R\&R resampler, both at
$D=128$, one point per task. Points above the diagonal are DRR wins: 22 of 24
tabular datasets and all eight gene-expression tasks.}
\label{fig:dom}
\end{figure}

The size of this advantage varies across tasks. Over the 32 tasks the per-task
$\Delta\mathrm{AP}$ is positively associated with the feasibility floor
(Spearman $+0.73$ at $D=0$ and $+0.75$ at $D=128$), whereas its association with
$\ESS(\what)/n_0$ has the expected negative sign (a smaller $\ESS(\what)/n_0$
means more donor reuse) but is weak ($-0.12$ and
$-0.08$).  Figure~\ref{fig:mech} shows the plot.

\begin{figure}[t]
\centering
\includegraphics[width=0.98\linewidth]{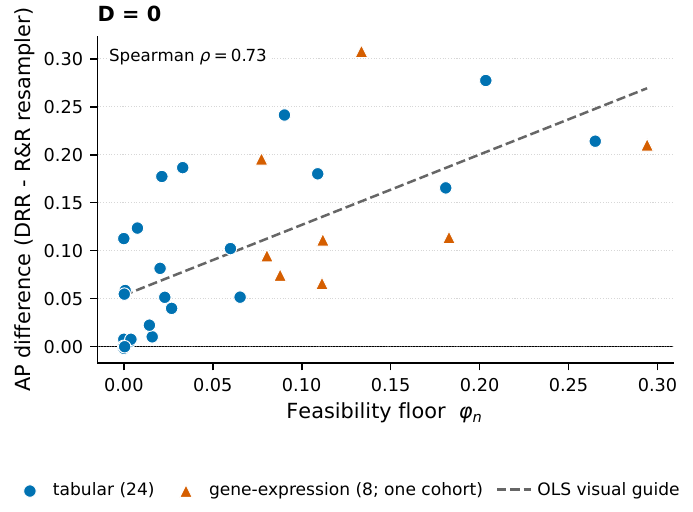}
\caption{Per-task advantage of DRR over the R\&R resampler
($\Delta\mathrm{AP}$, $D{=}0$) against the feasibility floor $\varphi_n$. The
Spearman correlation is $+0.73$ across $32$ label tasks. The dashed line is a
least-squares visual guide.}
\label{fig:mech}
\end{figure}

\subsection{Performance across the benchmark panel}
Table~\ref{tab:panel} compares ranking and probability accuracy across the
tabular benchmark panel. DRR$(D{=}128)$ has the highest mean AP ($0.455$, against
$0.440$ for the highest-scoring baseline), whereas balanced random forest has
the highest mean ROC-AUC ($0.864$, against $0.858$ for DRR). Multi-learner rows
average SVM, RF, NB, KLR, and MLP, except cost-sensitive learning, which uses
SVM/RF/KLR. Balanced RF, RUSBoost, and HistGB are standalone classifiers;
learner-specific results are reported in \appref{app:percls}.

\begin{table}[!htb]
\centering
\caption{Available-case mean performance on the tabular benchmarks over 30
trials. AP and AUC measure ranking; Brier measures probability accuracy.
Best values are bold (lower is better for Brier).}
\label{tab:panel}
\normalsize
\setlength{\tabcolsep}{6pt}
\begin{tabular}{@{}lcccc@{}}
\toprule
& & & \multicolumn{2}{c}{Brier}\\
\cmidrule(l){4-5}
Method & AP & AUC & reported & 1-Platt\\
\midrule
base (\texttt{none}) & 0.422 & 0.832 & 0.0588 & 0.0425\\
cost-sensitive & 0.440 & 0.855 & 0.0615 & 0.0416\\
logit adjust & 0.421 & 0.831 & -- & 0.0361\\
\texttt{SMOTE} & 0.415 & 0.837 & 0.0828 & 0.0408\\
\texttt{SMOTE-ENN} & 0.402 & 0.843 & 0.1087 & 0.0399\\
\texttt{FW-SMOTE} & 0.426 & 0.839 & 0.0642 & 0.0410\\
balanced RF & 0.424 & \textbf{0.864} & 0.1513 & 0.0412\\
RUSBoost & 0.369 & 0.823 & 0.1533 & 0.0437\\
R\&R-resample$(D{=}128)$ & 0.364 & 0.840 & 0.1616 & 0.0431\\
HistGB & 0.427 & 0.843 & 0.0355 & 0.0417\\
dual score alone $(D{=}0)$ & 0.412 & 0.839 & 0.0371 & 0.0371\\
dual score alone $(D{=}128)$ & 0.450 & 0.856 & \textbf{0.0347} & \textbf{0.0347}\\
\textbf{DRR}$(D{=}0)$ & 0.450 & 0.856 & 0.0352 & 0.0352\\
\textbf{DRR}$(D{=}128)$ & \textbf{0.455} & 0.858 & 0.0349 & 0.0349\\
\bottomrule
\end{tabular}
\par\smallskip
\begin{minipage}{\linewidth}
\footnotesize\setlength{\baselineskip}{10pt}
$N=22$ datasets for SMOTE-ENN and $N=24$ otherwise, including FW-SMOTE.
Reported Brier uses each pipeline's output; 1-Platt applies one held-out Platt
map per pipeline. DRR and the dual score alone already include this map, so
their two entries coincide. The dual score does not depend on the base learner,
so its rows are one value per dataset; they come from the fixed-weight ablation
run. Boldface uses unrounded means. Dataset-bootstrap 95\% intervals for these
means have half-widths of about $0.09$--$0.12$ (AP) and $0.04$--$0.05$ (AUC)
for every row, reflecting between-dataset heterogeneity; the paired comparisons
with intervals are in \appref{app:significance}.
\end{minipage}
\end{table}

\par\medskip\noindent
Reported-output Brier reflects each pipeline's probability mapping and
training distribution. Native Brier is $0.15$--$0.16$ for balanced RF,
RUSBoost, and the R\&R resampler, whereas HistGB ($0.0355$) and DRR ($0.0349$) are both
fitted at the original prior and lie within $0.001$ of each other; the
reliability curves 
in Figure~\ref{fig:calib}{} show the native outputs. The 1-Platt column
applies one held-out Platt map to every pipeline. It closes most of the gap for
the rebalancing arms ($0.040$--$0.044$) but worsens HistGB
($0.0355\to0.0417$), whose native output is already well calibrated. After this
matched treatment the dual score alone and DRR have the lowest Brier ($0.0347$
and $0.0349$); DRR's Brier is below cost-sensitive learning (Hodges--Lehmann
difference $-0.0046$, 95\% CI $[-0.0108,-0.0014]$) and the R\&R resampler
($-0.0069$, $[-0.0135,-0.0012]$) but not distinguishable from logit adjustment
($-0.0010$, $[-0.0022,0.0000]$).
\par\medskip

\begin{figure}[t]
\centering
\includegraphics[width=0.78\linewidth]{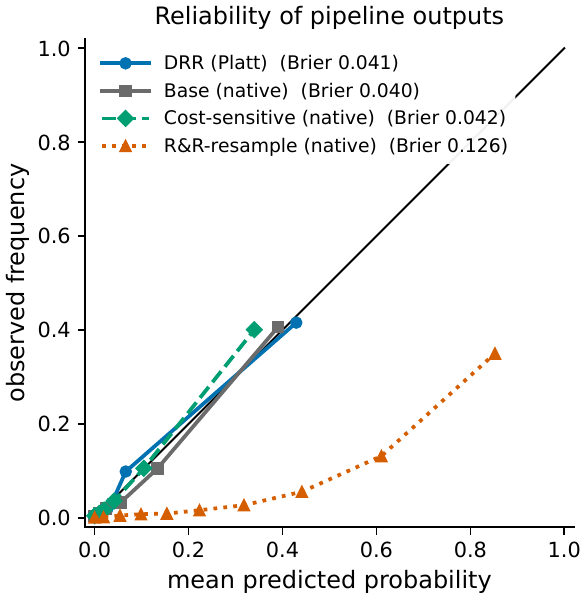}
\caption{Reliability diagram pooled over four illustrative datasets
(\texttt{oil}, \texttt{glass4}, \texttt{yeast}, \texttt{us\_crime};
random-forest base). Test predictions from $20$ repeated half splits are pooled
across splits and divided into up to ten probability-quantile bins;
tied discrete probabilities are left unsplit, so fewer bins may be occupied.
DRR is shown after its pipeline Platt map, whereas the other curves show native
outputs. Matched one-Platt Brier results are given in Table~\ref{tab:panel}; the
legend reports Brier scores for the displayed outputs.}
\label{fig:calib}
\end{figure}

The fixed published FW-SMOTE configuration has mean AP $0.426$.
HistGB scores
$0.427$ AP and $0.843$ AUC. DRR's mean AP advantage over it is $0.0286$, but
the paired difference is not significant (Holm $p=0.303$) and the per-dataset
picture is heterogeneous (\appref{app:perdataset}): among the four tabulated
methods HistGB is best on 9 datasets and DRR on 10. HistGB's largest wins over
DRR (\texttt{arrhythmia}, $0.724$ against $0.251$, and \texttt{compustat},
$0.305$ against $0.218$) are of the same order as DRR's over HistGB
(\texttt{yeast}, $0.621$ against $0.069$, and \texttt{glass4}, $0.707$ against
$0.448$), so the mean difference reflects a few large wins on each
side, not uniform superiority. A class-weighted HistGB variant is reported in
\appref{app:repro}.

Restricting the multi-learner arms to their common SVM/RF/KLR bases gives mean
AP $0.460$ for DRR, $0.440$ for cost-sensitive learning, and $0.378$ for the
R\&R resampler over the $24$ tabular datasets; the paired tests in
\appref{app:significance} pair DRR's five-learner dataset means with
cost-sensitive learning's three-learner means, and the learner-matched values
show that the ordering does not depend on this choice. SMOTE-ENN has no valid
rows on \texttt{glass4} and \texttt{phoss}; FW-SMOTE completes all $24$.
Failures are omitted, not imputed, and on the $22$ datasets
common to DRR and SMOTE-ENN the mean AP difference is $+0.0554$, against
$+0.0539$ available-case.

The operating-point metrics --- balanced accuracy, $F_1$, MCC and G-mean ---
depend on the cut, so we report the three cut policies of \S\ref{sec:proto}
(full tables and a summary figure in \appref{app:cuts}). At the fixed $t=0.5$
DRR trails several rebalancing methods. After threshold tuning DRR has the highest mean on
every operating-point metric under either objective. The margins over
cost-sensitive learning and balanced random forest are small, however: at most
$0.005$ on balanced accuracy and G-mean (for example balanced accuracy $0.789$
against $0.788$) and between $0.004$ and $0.027$ on MCC and $F_1$, none
significant, with paired win/loss counts near even
(\appref{app:significance}). The clear operating-point advantage is over the
resampling arms, RUSBoost, the base, and logit adjustment.

The average-rank, or critical-difference, diagram on average precision
(Figure~\ref{fig:cd}) gives DRR$(D{=}128)$ the best mean rank. The Nemenyi
comparison separates it from SMOTE, SMOTE-ENN, and the R\&R resampler at
$\alpha=0.05$; its rank gap from FW-SMOTE ($3.23$) is below the critical
difference ($3.55$). In the separate eleven-method complete-case family, the
Friedman omnibus test rejects equality on every discrimination metric
($\chi^2(10)=46.19$--$71.67$, all $p<10^{-5}$). The Holm-corrected
control-versus-rest comparisons in \appref{app:significance} show significant
improvements over all four data-level resampling arms, \texttt{RUSBoost}, the
base, and logit adjustment on AP, ROC-AUC, balanced accuracy, $F_1$, MCC and
G-mean, with threshold-dependent metrics at the tuned cuts. Differences from
cost-sensitive learning, balanced random forest, and HistGB are not significant
on any of these metrics; the paired interval estimates in
\appref{app:significance} bound these differences (for AP, upper limits of
$0.03$--$0.06$).

\begin{figure}[t]
\centering
\includegraphics[width=0.98\linewidth]{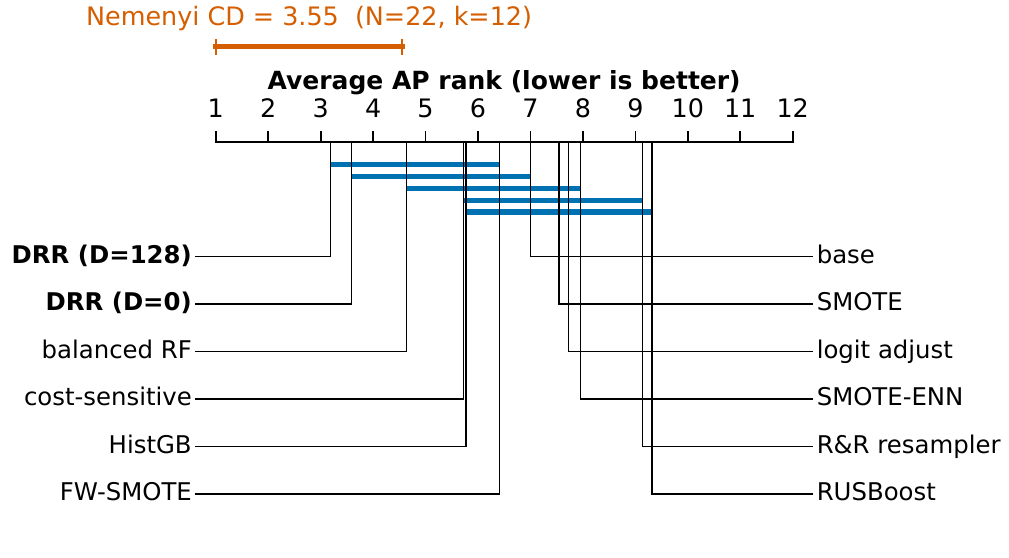}
\caption{Critical-difference diagram of average ranks on average precision over
the $22$ complete tabular benchmarks. Methods joined by a bar are not
significantly different at $\alpha=0.05$ (Nemenyi, $12$ methods, critical
difference $3.55$). Lower ranks are better. This all-pairs analysis includes both
DRR resolutions and is separate from the eleven-method control-versus-rest
family.}
\label{fig:cd}
\end{figure}

\subsection{Dataset- and learner-level analyses}
Per-dataset distributions ( Figure~\ref{fig:box}) place the DRR variants at or near the
top of the median and interquartile summaries for AP and $F_1$.

\begin{figure}[!htbp]
\centering
\includegraphics[width=0.98\linewidth]{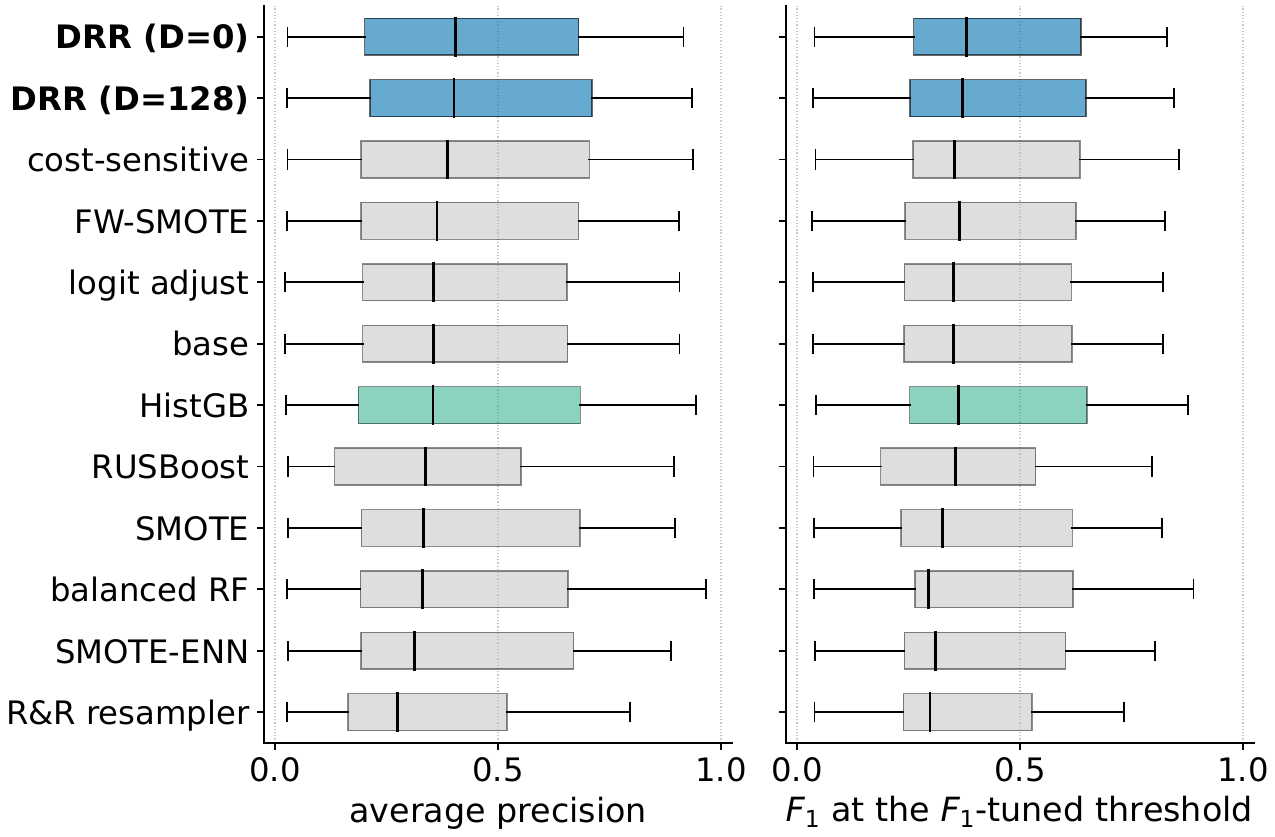}
\caption{Per-dataset distributions of average precision (left) and $F_1$ (right)
over the tabular benchmarks, one value per available dataset. Boxes show the
interquartile range and median; whiskers extend to the most extreme observations
within 1.5 interquartile ranges of the box; observations beyond the whiskers
are shown individually. Methods are sorted by median AP.}
\label{fig:box}
\end{figure}

By base classifier, DRR$(D{=}128)$ has higher mean AP than the base,
cost-sensitive learning where defined, and the R\&R resampler in each reported
base-classifier column.
For naive Bayes, the dual term lifts average precision from $0.339$ to
$0.423$, but the dual score alone reaches $0.450$ under every learner
(Table~\ref{tab:panel}): with a base this weak, and likewise for kernel
logistic regression ($0.437$ against $0.450$), the fixed $w=\tfrac12$
under-weights the dual, whereas under the three stronger bases fusion improves
on the dual score alone by $0.020$--$0.026$ (15--16 wins of 24). The
per-classifier averages are reported in \appref{app:percls}. The per-dataset results in
\appref{app:perdataset}\ list, for each benchmark, the average precision of
DRR, cost-sensitive learning,
HistGB, and the R\&R resampler, together with the imbalance ratio and the
feasibility floor, ordered by the floor.

\section{Discussion}\label{sec:discussion}
DRR turns the survey-calibration dual that raking already computes into a
score and adds it, at a fixed weight, to a classifier trained at the original
prior. On 24 tabular benchmarks it
improves average precision over its base on every dataset at $D=128$, exceeds
the R\&R resampler built from the same dual on 22, and has the highest mean
average precision in Table~\ref{tab:panel}. The shared-dual comparison
demonstrates the value of using the fitted direction directly for ranking:
the gain is achieved while preserving the base classifier and its original
training data.

The method also gives raking a direct interpretation as a classification
score. Under the log-linear tilt model and exact population matching, the
dual recovers the log density ratio up to an additive constant. The fusion
analysis then identifies how the two scores' separation and correlation
determine the benefit of combining them. The empirical ablation illustrates
two useful regimes: at $D=0$, fusion contributes most of the gain; at
$D=128$, the dual supplies a strong score and fusion extends improvement over
the base from 17 to all 24 datasets. With the SVM, random-forest, and MLP
bases, fusion improves mean AP over the dual alone by $0.020$--$0.026$.

DRR can be applied to an existing classifier at modest additional cost. The
dual requires one convex solve per fit subset, with a median runtime under
one second in our experiments (\appref{app:repro}), and the same fitted
direction is reused across base learners. Prediction requires score
evaluation and standardization, without classifier refitting. The
supplementary sweep finds the gain over the base stable for $D\ge64$ and
$\eta\le0.2$ (\appref{app:sensitivity}), supporting the use of a fixed
configuration across the benchmark panel.

The present evaluation concerns fixed-weight binary classification, with
the gene-expression tasks treated one-versus-rest. The R\&R comparison uses
the uncapped implementation described in \appref{app:repro}, and the
separation theory's AUC and AUPRC consequences use the bi-normal model
(Remark~\ref{rem:df}). Cross-fitted selection of the fusion weight is a
promising extension, particularly for weaker base scores where larger dual
weights improve ranking in the sensitivity analysis. Applying DRR to boosted
classifiers and joint multiclass prediction would extend its use to further
classification settings.

\section*{Data availability}
The 16 inherited tabular CSV files are copies of public benchmark data sets
used in the earlier R\&R study, retained unmodified; eight additional tabular benchmarks were retrieved from the
imbalanced-learn Zenodo collection (version 1, doi:10.5281/zenodo.61452). The
eight gene-expression tasks are the OpenML version-1 datasets identified in
\appref{app:repro}. The code and trial-level outputs supporting the tables and
figures will be made publicly available upon journal publication and are available from
the corresponding author on request. The trial-level outputs identify the dataset,
trial, base classifier, method, and applicable cut and model settings, so
table entries and their dispersion are recomputable. Dataset names and label rules are documented in
\appref{app:repro}.

\section*{CRediT authorship contribution statement}
\textbf{Dongha Kim}: Conceptualization, Methodology, Software, Investigation,
Writing -- original draft. \textbf{Seunghwan Park}: Conceptualization,
Methodology, Formal analysis, Writing -- review \& editing.

\section*{Declaration of competing interest}
The authors report financial support from the National Research Foundation of
Korea. They declare no other competing financial interests or personal
relationships that could have appeared to influence the work reported in this
paper.

\section*{Funding}
This work was supported by the National Research Foundation of Korea (NRF)
grants funded by the Korean government (MSIT): Nos.\ RS-2025-24683613
and RS-2026-25496805 to Dongha Kim, and No.\ RS-2026-25489716 to
Seunghwan Park.
The funders provided financial support only and had no role in study design;
data collection, analysis, or interpretation; manuscript preparation; or the
decision to submit the article for publication.

\section*{Declaration of generative AI and AI-assisted technologies in the
manuscript preparation process}
The authors used OpenAI Codex to assist with English-language editing,
sentence-level and structural revisions, and preparation of the plotting code
for the figures.
The authors reviewed and edited all AI-assisted content, independently verified
the scientific content, and take full responsibility for the content of the article.

\appendix

\section{Reproducibility details}\label{app:repro}
All experiments, analyses, and figures were produced by Python scripts run by
the authors. OpenAI Codex assisted with preparing the plotting and analysis
scripts. The authors reviewed the code; numerical outputs were checked against
saved predictions, protocol specifications, and recorded results, and plotted
values were checked against the underlying results or formulas.

All experiments use master seed $20240725$. Each of the $30$ trials uses a
stratified $50/50$ train--test split. The training half is split, with
stratification, $70/30$ into fit and threshold subsets using seed
$20240725+\text{trial}$.
Encoders, scalers, feature maps, marginal score standardizations, thresholds, and
Platt maps are fitted without access to the test half. Numeric variables are
retained; categorical variables are one-hot encoded with previously unseen levels
ignored. No imputation or univariate feature selection is applied.

Base learners use the following settings, fixed in advance and not tuned per
dataset; this uniformity aids comparability but need not be optimal for every
benchmark. Arguments not listed retain the
defaults of the archived environment. The SVM uses a standardized RBF kernel and probability
estimates; the random forest has $100$ trees; mixed naive Bayes uses Gaussian
continuous blocks and categorical blocks with smoothing $\alpha=1$; kernel logistic
regression uses standardization, $100$ RBF Nystr\"om components, and logistic
regression with at most $2{,}000$ iterations; and the tabular-only MLP has hidden
layers $(64,32)$ and at most $500$ iterations. The cost-sensitive arm sets
\texttt{class\_weight=balanced} for SVM, random forest, and kernel logistic
regression. Balanced random forest uses $100$ trees and RUSBoost $50$ estimators.
The standalone HistGB model uses learning rate $0.05$, at most $300$ iterations,
$31$ leaves, minimum leaf size $10$, $\ell_2$ regularization $1$, and early stopping
with validation fraction $0.15$, patience $20$, and tolerance $10^{-7}$; its
class-weighted variant (AP $0.406$, AUC $0.830$, Brier $0.071$ native and
$0.040$ after one Platt map) is a secondary ablation outside the main
inferential family.

The FW-SMOTE arm implements the published equations and Algorithm~1 of
\citet{fwsmote}, with the basic RIM quantifier $Q(t)=t^{0.4}$, Minkowski
exponent $p=2$, and all encoded coordinates retained. Fit-subset Fisher scores
use the absolute difference of class means divided by the sum of the two
sample variances. Coordinates sorted by decreasing Fisher score receive
weights $q_j=(j/d)^{0.4}-((j-1)/d)^{0.4}$, where $d$ is the encoded dimension;
distance is $\{\sum_jq_j(x_j-x'_j)^2\}^{1/2}$. Ties retain coordinate order.
Each minority anchor selects three distinct neighbors among its five nearest
other minority observations and generates one uniform linear interpolation per selected
neighbor. Small fit splits use $k=\min(5,n_{1,\mathrm{fit}}-1)$ and
$N=\min(3,k)$. Thus this arm adds $N$ synthetic positives per original
positive rather than enforcing a balanced training prior. This is a fixed
published configuration, not a reproduction of the authors' parameter-tuned
optimum. The archived Fisher-weighted SMOTE adaptation, which directly used
squared Fisher scores and oversampled to balance, is retained in the raw
archive but is replaced by this arm in the manuscript's FW-SMOTE comparisons.
All 720 FW-SMOTE splits complete. Twenty MLP fits reach the fixed 500-iteration
limit and are retained under the common stopping rule. On all 30
\texttt{glass4} splits, the Nystr\"om component count is automatically capped
by the available training observations. These recorded warnings do not
produce missing prediction rows.

For DRR, $D=0$ uses the fitted linear/one-hot coordinates, whereas $D=128$
also includes a 128-dimensional Gaussian RFF block when continuous variables
are present. The linear block is standardized using fit-subset means and
standard deviations, then divided by its largest Euclidean norm on that subset
(with denominator floor $10^{-12}$). The one-hot block is divided by the square
root of the number of categorical variables. The RFF block uses the
$\sqrt{2/D}\cos(\cdot)$ convention. Each active block is then multiplied by
$1/\sqrt{B}$, where $B$ is the number of active blocks
(\texttt{block\_scaling=equal}); categorical interaction blocks are disabled.
Thus adding an RFF block also changes the common block multiplier.
The RFF block uses a kernel with
\[
 \gamma_{\mathrm{rff}}=\{2\,\operatorname{median}_{i<j}\|z_i-z_j\|^2\}^{-1},
\]
estimated from at most $1{,}000$ fit observations (subsampling seed $0$); the RFF
draw uses seed $20240725$. The mathematical tolerance rule is
\[
 \delta_n=\max\{\varphi_n+0.05(\gamma_n-\varphi_n),\,1.05\varphi_n\}.
\]
In implementation, $\varphi_n$ is replaced by the achieved residual upper bound
$U_n$ from pairwise Frank--Wolfe, and $\gamma_n-U_n$ is clipped below at zero.
The routine allows at most 20,000 iterations, with relative tolerance $10^{-8}$
and absolute numerical floor $10^{-12}$ in its duality-gap stopping rule.
If the
resulting tolerance reaches $\gamma_n$, the solver returns zero dual and the
base weights. For the exact-norm problem, multiplying the entire feature map by
a common positive constant scales $\varphi_n$, $\gamma_n$, and $\delta_n$ together
and leaves the tilt weights and dual scores unchanged; this invariance does not
extend to arbitrary relative rescaling of the blocks.

The fixed-grid weight sensitivity analysis uses $w\in\{0,0.1,0.25,0.5,1\}$. The weight-sensitivity figure bootstraps the
24 paired dataset-level AP differences from $w=0$, with 9,999 resamples and
pointwise percentile 95\% intervals, using the raw-score results.
The dense dual uses damped Newton
iterations for at most $200$ steps with nominal gradient tolerance $10^{-9}$
($10^{-6}$ fallback on a line-search stall or the iteration limit), norm smoothing
$10^{-10}$, and numerical ridge $10^{-10}$ when the feature dimension is at most
$600$; otherwise it uses L-BFGS-B. The final fused score is mapped to a
probability by logistic Platt scaling on the threshold subset. AP and ROC-AUC are
computed from the raw pre-calibration score; Brier and thresholded predictions
use the Platt probability. The Platt slope is an unconstrained logistic
coefficient and can be nonpositive on a weak threshold split, in which case the
map reverses the ranking; this occurred for $1.3$--$1.4\%$ of the $3{,}600$
fitted DRR maps at each $D$ (and $2.2\%$ of all $54{,}000$ score units in the
ablation run), which is why AP and ROC-AUC use the raw score.
Logit adjustment applies the prior shift to probabilities clipped at $10^{-9}$,
which ties probabilities saturated at $0$ or $1$ (frequent for naive Bayes and
the MLP) and moves its AP and AUC by about $0.001$ relative to \texttt{none}.
In the archived runs, solving the dual took a median of $0.4$ s ($D=0$) and
$0.5$ s ($D=128$) per fit subset, with maxima of $167$ s and $163$ s on
\texttt{phoss}, the highest-dimensional categorical design ($p=480$); the dual
is solved once per fit subset and reused across base learners, so DRR adds no
classifier refit.

For nearest-neighbor samplers we set
$k=\max\{1,\min(5,n_{1,\mathrm{fit}}-1)\}$ whenever the implementation exposes the
neighbor count; the SMOTE-ENN arm uses its archived-library default. The R\&R
resampler uses $m=\lfloor n_{\mathrm{fit}}/4\rfloor$ and PPS sampling with
replacement; following the R\&R convention, the added pseudo-minority block is
split evenly between relabeled majority donors and bootstrapped minority units
(the code's \texttt{rho}$=1/2$, unrelated to the within-class correlation
$\rho$ of the fusion analysis). It records but does not enforce the diagnostic
condition $m\le\ESS(\what)$ ($\kappa_{\ESS}=1$).

The cap would bind in most of the panel: with $m=\lfloor n_{\mathrm{fit}}/4\rfloor$,
$m$ exceeds $\ESS(\what)$ in 91\% of the dataset--trial cells at $D=0$ and 92\%
at $D=128$ (median $m/\ESS(\what)$ of $3.0$ and $2.7$). A separate $D=0$
diagnostic on the same splits (SVM and random-forest bases; 24 datasets, 30
trials) compared an uncapped anchor draw of $m=\lfloor n_0/4\rfloor$ --- a
majority-scale draw size deliberately different from the comparator's --- with
the caps $m\le\tfrac12\ESS$, $\ESS$, and $2\ESS$, which bind in 97\%, 90\%, and
68\% of cells and reduce the mean number of distinct relabeled donors from $167$
to $58$, $96$, and $132$. Enforcing the caps lowers the resampler's mean AP by
$0.071$, $0.033$, and $0.012$ (dataset-paired; 22, 22, and 17 losses of 24;
Holm-adjusted $p=0.0003$, $0.0003$, and $0.0016$) and raises the
redraw-to-redraw dispersion of its AP (root-mean-square standard deviation
over redraws $0.034$ uncapped against $0.054$, $0.047$, and $0.039$). Because a
cap also changes the training mass and the effective regularization of the
retrained classifier, this is a comparator diagnostic rather than an estimate
of a pure sampling-noise effect. It favors the uncapped draw under these
$D=0$, SVM/random-forest, and anchor-draw conditions; it does not establish
the effect of capping under the main $D=128$ comparison.

The archived environment
records scikit-learn 1.3.0 and imbalanced-learn 0.11.0; a complete environment lock
will be released with the code.

Table~\ref{tab:datasets} lists the 24 tabular benchmarks with their sample
sizes, attribute counts, and imbalance ratios; \texttt{abalone} is the
$2{,}338$-row KEEL subset of the UCI abalone table with ring class 17 as the
positive class, and \texttt{abalone\_19} is the full $4{,}177$-row table with
ring class 19 positive. The gene-expression tasks are the
eight OpenML version-1 one-versus-rest views in
Table~\ref{tab:ova}; identifiers and label definitions are taken from the
corresponding OpenML metadata. They share the same $1{,}545$ rows and $10{,}935$
predictors and are therefore distinct label tasks on one cohort, not eight
independent cohorts. The underlying GEMLeR collection derives from the
expO/International Genomics Consortium data (GEO accession GSE2109), using
Affymetrix U133 Plus 2.0, MAS5-normalized, variance-filtered preprocessing
 \citep{stiglic2010}. The
benchmark files do not provide batch labels,
and the \texttt{ID\_REF} field is ignored; splits are consequently row-level rather
than batch-grouped.

\begin{table}[t]
\centering
\normalsize
\setlength{\tabcolsep}{4pt}
\caption{The 24 tabular benchmarks: sample size $n$, number of raw attributes
$p$ and its split into continuous and categorical variables, minority count
$n_1$, and imbalance ratio $\mathrm{IR}=n_0/n_1$ on the full sample, as loaded
by the experiment drivers before encoding. The fit-subset ratio used in
Table~\ref{tab:perdataset} differs slightly through stratified rounding.}
\label{tab:datasets}
\begin{tabular}{@{}lrrrrrr@{}}
\toprule
dataset & $n$ & $p$ & continuous & categorical & $n_1$ & IR\\
\midrule
\texttt{abalone} & 2{,}338 & 8 & 7 & 1 & 58 & 39\\
\texttt{abalone\_19} & 4{,}177 & 10 & 10 & 0 & 32 & 130\\
\texttt{arrhythmia} & 452 & 278 & 278 & 0 & 25 & 17\\
\texttt{boundary} & 3{,}505 & 175 & 0 & 175 & 123 & 27\\
\texttt{cam} & 18{,}916 & 132 & 0 & 132 & 942 & 19\\
\texttt{car-good} & 1{,}728 & 6 & 0 & 6 & 69 & 24\\
\texttt{compustat} & 13{,}657 & 20 & 20 & 0 & 520 & 25\\
\texttt{flare-f} & 1{,}066 & 11 & 7 & 4 & 43 & 24\\
\texttt{glass4} & 214 & 9 & 9 & 0 & 13 & 15\\
\texttt{hypo} & 2{,}068 & 19 & 5 & 14 & 122 & 16\\
\texttt{isolet} & 7{,}797 & 617 & 617 & 0 & 600 & 12\\
\texttt{led} & 443 & 7 & 0 & 7 & 37 & 11\\
\texttt{libras\_move} & 360 & 90 & 90 & 0 & 24 & 14\\
\texttt{mammography} & 11{,}183 & 6 & 6 & 0 & 260 & 42\\
\texttt{oil} & 937 & 46 & 46 & 0 & 41 & 22\\
\texttt{phoss} & 11{,}411 & 480 & 0 & 480 & 613 & 18\\
\texttt{scene} & 2{,}407 & 294 & 294 & 0 & 177 & 13\\
\texttt{sick} & 3{,}621 & 21 & 1 & 20 & 225 & 15\\
\texttt{spectrometer} & 531 & 93 & 93 & 0 & 45 & 11\\
\texttt{us\_crime} & 1{,}994 & 100 & 100 & 0 & 150 & 12\\
\texttt{vowel0} & 988 & 12 & 10 & 2 & 90 & 10\\
\texttt{winequality} & 656 & 11 & 11 & 0 & 18 & 35\\
\texttt{yeast} & 482 & 8 & 8 & 0 & 20 & 23\\
\texttt{yeast\_ml8} & 2{,}417 & 103 & 103 & 0 & 178 & 13\\
\bottomrule
\end{tabular}
\end{table}

\begin{table}[t]
\centering
\normalsize
\caption{OpenML/GEMLeR one-versus-rest gene-expression tasks. All entries are
version 1, use the same $1{,}545$ samples and $10{,}935$ predictors after
\texttt{ID\_REF} is removed, and contain no missing values. The named tissue is
class 1 and \texttt{Other} is class 0.}
\label{tab:ova}
\begin{tabular}{@{}lrrr@{}}
\toprule
task & OpenML ID & $n_1$ & $n_0/n_1$\\
\midrule
\texttt{OVA\_Breast} & 1128 & 344 & 3.49\\
\texttt{OVA\_Colon} & 1161 & 286 & 4.40\\
\texttt{OVA\_Endometrium} & 1142 & 61 & 24.33\\
\texttt{OVA\_Kidney} & 1134 & 260 & 4.94\\
\texttt{OVA\_Lung} & 1130 & 126 & 11.26\\
\texttt{OVA\_Ovary} & 1166 & 198 & 6.80\\
\texttt{OVA\_Prostate} & 1146 & 69 & 21.39\\
\texttt{OVA\_Uterus} & 1138 & 124 & 11.46\\
\bottomrule
\end{tabular}
\end{table}

\section{First-stage PPS variance}\label{app:pps}
For completeness, we give the idealized result summarized in the main text. Use
the same majority-side sample and raking solution in both channels. Conditional
on the observed data, hold these quantities fixed for the PPS draw; across
repeated datasets they may vary. We compare two estimators of the target minority mean
$\mu_1=\mathbb{E}[\phik(X)\mid Y=1]$. Let
$\widehat\mu_1=n_1^{-1}\sum_{i:Y_i=1}\phik(X_i)$, and let
$\widetilde\mu_1$ be the unweighted mean of $m$ majority embeddings drawn
independently with replacement with probabilities $\what$. Assume i.i.d.
class-conditional samples, finite class-conditional second moments, and an RNG
conditionally independent of the observed data. For the across-dataset comparison
below, treat the feature map as fixed rather than estimated from the same sample.
Conditional on the data,
\[
 \mathbb{E}(\widetilde\mu_1\mid\mathrm{data})=\Phi^\top\what,
 \qquad
 \operatorname{Var}(\widetilde\mu_1\mid\mathrm{data})=\Sigma_{\what}/m,
\]
where $\Sigma_{\what}$ is the $\what$-weighted covariance operator and
$\Sigma_1=\operatorname{Cov}\{\phik(X)\mid Y=1\}$. Under exact matching (which
requires $\varphi_n=0$), $\Phi^\top\what=\widehat\mu_1$, and the law of total
variance gives
\[
 \operatorname{Var}(\widetilde\mu_1)
 =\Sigma_1/n_1+\mathbb{E}(\Sigma_{\what}/m).
\]
Thus the PPS channel contributes a positive-semidefinite term absent from the
plug-in mean; it is strict only in directions with positive weighted variance.
If an ESS rule $m\le\kappa_{\ESS}\ESS(\what)$ is enforced, then
$\Sigma_{\what}/m\succeq\Sigma_{\what}/\{\kappa_{\ESS}\ESS(\what)\}$ for that
weight vector; both $\Sigma_{\what}$ and $\ESS$ change with $\what$, so this is
not a monotonicity statement across weights.

\begingroup\clubpenalty=10000
At the deployed positive tolerance and a nonzero exact-norm solution,
stationarity gives
$\Phi^\top\what=\widehat\mu_1-\delta_n\thetahat/\|\thetahat\|$.
The displacement is relative to the plug-in target. With norm smoothing its
magnitude is $a_n=\delta_n\|\thetahat\|/\sqrt{\|\thetahat\|^2+\varepsilon_s^2}$,
subject to solver tolerance. Neither identity bears on the AP of the retrained
classifier; that comparison is empirical (\S\ref{sec:results}).
\par\endgroup

We audited the first displayed conditional variance directly on 34 distinct
fixed states spanning three datasets, two trials, both $D=0$ and $D=128$, and
three draw-size policies, using 1,000 Monte Carlo draws per state. All cells met
the prespecified Monte Carlo tolerance; the empirical-to-theoretical trace ratio
ranged from $0.972$ to $1.059$ and averaged $1.001$.

\section{Per-dataset results}\label{app:perdataset}
Table~\ref{tab:perdataset} gives the per-dataset average precision behind the
aggregate numbers and the diagnostic analysis in the main manuscript. For each of the $24$
tabular benchmarks it lists the imbalance ratio, the feasibility floor
$\varphi_n$, and the average precision of DRR$(D{=}128)$, cost-sensitive learning,
HistGB, and the R\&R resampler, with the DRR-minus-R\&R-resampler margin in the
last column. Datasets are ordered by decreasing $D=128$ floor. Larger margins
cluster near the top of this ordering, and the two losses to the R\&R resampler
(\texttt{yeast\_ml8}, \texttt{abalone\_19}) are both about $0.001$ in magnitude and sit among
low-floor rows, although low-floor tasks can also show sizable positive margins
(\texttt{mammography} $+0.128$, \texttt{compustat} $+0.061$). This is the rank
association reported in \S\ref{sec:results}.

\begin{table}[t]
\centering
\normalsize
\setlength{\tabcolsep}{2pt}
\renewcommand{\arraystretch}{1}
\caption{Per-dataset raw-score average precision of DRR$(D{=}128)$,
cost-sensitive learning, HistGB, and the R\&R resampler (R\&R-resample, at
$D=128$ here and in the following tables) over the 24 tabular benchmarks, with
imbalance ratio (IR) and the $D=128$ feasibility floor $\varphi_n$, averaged
over 30 trial splits. Datasets are sorted by decreasing floor; the best of the
four methods per row is in bold. The final column is the unrounded
DRR-minus-R\&R-resampler margin.}
\label{tab:perdataset}
\begin{tabular}{@{}lrrccccc@{}}
\toprule
dataset & IR & $\varphi_n$ & DRR & \shortstack{cost-\\sens.} & HistGB & \shortstack{R\&R-\\resample} & $\Delta$AP\\
\midrule
\texttt{yeast} & 23 & 0.387 & \textbf{0.621} & 0.473 & 0.069 & 0.377 & $+0.243$\\
\texttt{glass4} & 15 & 0.310 & 0.707 & \textbf{0.782} & 0.448 & 0.515 & $+0.191$\\
\texttt{spectrometer} & 11 & 0.249 & 0.869 & \textbf{0.899} & 0.795 & 0.675 & $+0.193$\\
\texttt{libras\_move} & 15 & 0.231 & \textbf{0.808} & 0.798 & 0.646 & 0.512 & $+0.296$\\
\texttt{vowel0} & 10 & 0.156 & 0.934 & 0.938 & \textbf{0.944} & 0.763 & $+0.171$\\
\texttt{arrhythmia} & 18 & 0.097 & 0.251 & 0.331 & \textbf{0.724} & 0.155 & $+0.096$\\
\texttt{oil} & 22 & 0.092 & \textbf{0.495} & 0.442 & 0.403 & 0.298 & $+0.197$\\
\texttt{winequality} & 37 & 0.081 & 0.145 & 0.147 & \textbf{0.156} & 0.139 & $+0.006$\\
\texttt{hypo} & 16 & 0.074 & \textbf{0.847} & 0.834 & 0.798 & 0.796 & $+0.051$\\
\texttt{led} & 11 & 0.065 & \textbf{0.723} & 0.679 & 0.709 & 0.671 & $+0.051$\\
\texttt{mammography} & 42 & 0.053 & 0.662 & 0.597 & \textbf{0.676} & 0.534 & $+0.128$\\
\texttt{flare-f} & 24 & 0.048 & \textbf{0.242} & 0.200 & 0.217 & 0.212 & $+0.030$\\
\texttt{us\_crime} & 12 & 0.046 & \textbf{0.526} & 0.457 & 0.494 & 0.448 & $+0.078$\\
\texttt{isolet} & 12 & 0.035 & 0.898 & 0.878 & \textbf{0.908} & 0.717 & $+0.181$\\
\texttt{abalone} & 40 & 0.029 & \textbf{0.308} & 0.258 & 0.231 & 0.253 & $+0.055$\\
\texttt{scene} & 13 & 0.024 & 0.230 & \textbf{0.236} & 0.235 & 0.207 & $+0.023$\\
\texttt{yeast\_ml8} & 13 & 0.016 & 0.095 & 0.095 & 0.094 & \textbf{0.096} & $-0.001$\\
\texttt{boundary} & 28 & 0.016 & \textbf{0.109} & 0.084 & 0.081 & 0.099 & $+0.010$\\
\texttt{car-good} & 24 & 0.008 & 0.605 & 0.629 & \textbf{0.666} & 0.481 & $+0.123$\\
\texttt{phoss} & 18 & 0.004 & 0.172 & 0.170 & \textbf{0.190} & 0.166 & $+0.006$\\
\texttt{abalone\_19} & 132 & 0.003 & 0.026 & \textbf{0.028} & 0.025 & 0.027 & $-0.001$\\
\texttt{compustat} & 25 & 0.003 & 0.218 & 0.217 & \textbf{0.305} & 0.156 & $+0.061$\\
\texttt{sick} & 15 & 0.000 & \textbf{0.200} & 0.174 & 0.176 & 0.197 & $+0.003$\\
\texttt{cam} & 19 & 0.000 & 0.242 & 0.205 & \textbf{0.252} & 0.235 & $+0.008$\\
\midrule
mean & & 0.084 & \textbf{0.455} & 0.440 & 0.427 & 0.364 & $+0.092$\\
\bottomrule
\end{tabular}
\end{table}

\section{Significance tests}\label{app:significance}
Table~\ref{tab:significance} reports the post-hoc control-versus-rest comparisons
that accompany the Friedman omnibus test in the main manuscript. Each cell is the
Holm-corrected $p$-value of the Wilcoxon signed-rank test of DRR$(D{=}128)$
against the row method on the column metric, paired over the tabular benchmarks
and corrected within each metric across the ten baselines; the
threshold-dependent metrics are taken at each method's tuned cut. For
cost-sensitive learning, which is defined only for SVM, RF and KLR, the pairing
uses DRR's five-learner dataset means against cost-sensitive learning's
three-learner means; restricting DRR to the same three learners gives mean AP
$0.460$ and does not change the ordering. A cell is set in
bold when DRR is significantly better at $\alpha=0.05$, that is when the corrected
$p$-value is below $0.05$ \emph{and} DRR lies on the better side of the metric.
The family contains discrimination metrics only; a Brier comparison of DRR's
pipeline probability against baselines' native outputs would be asymmetric.
Matched one-Platt Brier results and paired probability-accuracy intervals are
reported in 
\S\ref{sec:results}. On the common $N=22$ datasets,
the eleven-method Friedman statistics (all with $10$ degrees of freedom) are
$57.567$ for AP ($p=1.0{\times}10^{-8}$),
$71.667$ for AUC ($p=2.1{\times}10^{-11}$),
$54.565$ for balanced accuracy ($p=3.8{\times}10^{-8}$),
$54.477$ for $F_1$ ($p=4.0{\times}10^{-8}$),
$46.192$ for MCC ($p=1.3{\times}10^{-6}$), and $48.429$ for G-mean ($p=5.2{\times}10^{-7}$). Thus the post-hoc tests are warranted throughout; the pairing
covers the $24$ benchmarks, or the $22$ on which SMOTE-ENN has complete results. Reading across, DRR is significantly better than seven of
the ten baselines --- the base, logit adjustment, all four data-level
resamplers, and \texttt{RUSBoost} --- on every one of the six metrics, while its
differences from cost-sensitive learning, balanced random forest, and HistGB do
not reach significance on any of them.

\begin{table}[t]
\centering
\normalsize
\setlength{\tabcolsep}{2pt}
\caption{Holm-corrected Wilcoxon signed-rank $p$-values for DRR$(D{=}128)$
against each baseline, per metric, paired over the tabular benchmarks;
threshold-dependent metrics at the tuned cut. Bold marks a comparison in which
DRR is significantly better at $\alpha=0.05$. Higher is better for every metric
shown; probability accuracy is compared separately in the main manuscript.}
\label{tab:significance}
\begin{tabular}{@{}lcccccc@{}}
\toprule
method & AP & AUC & bal.acc & $F_1$ & MCC & G-mean\\
\midrule
\texttt{none} & $\mathbf{1.2{\times}10^{-6}}$ & $\mathbf{2.2{\times}10^{-4}}$ & $\mathbf{3.0{\times}10^{-4}}$ & $\mathbf{2.7{\times}10^{-5}}$ & $\mathbf{9.7{\times}10^{-4}}$ & $\mathbf{4.8{\times}10^{-4}}$\\
cost-sensitive & $0.303$ & $0.747$ & $1.0$ & $0.236$ & $1.0$ & $1.0$\\
logit adjust & $\mathbf{1.2{\times}10^{-6}}$ & $\mathbf{1.6{\times}10^{-4}}$ & $\mathbf{3.0{\times}10^{-4}}$ & $\mathbf{1.7{\times}10^{-5}}$ & $\mathbf{1.0{\times}10^{-3}}$ & $\mathbf{4.4{\times}10^{-4}}$\\
\texttt{SMOTE} & $\mathbf{6.6{\times}10^{-5}}$ & $\mathbf{2.4{\times}10^{-4}}$ & $\mathbf{4.3{\times}10^{-4}}$ & $\mathbf{3.8{\times}10^{-4}}$ & $\mathbf{0.012}$ & $\mathbf{1.0{\times}10^{-3}}$\\
\texttt{SMOTE-ENN} & $\mathbf{3.3{\times}10^{-5}}$ & $\mathbf{1.5{\times}10^{-3}}$ & $\mathbf{7.7{\times}10^{-3}}$ & $\mathbf{1.4{\times}10^{-4}}$ & $\mathbf{0.013}$ & $\mathbf{0.010}$\\
\texttt{FW-SMOTE} & $\mathbf{6.3{\times}10^{-5}}$ & $\mathbf{5.3{\times}10^{-4}}$ & $\mathbf{2.6{\times}10^{-3}}$ & $\mathbf{3.8{\times}10^{-4}}$ & $\mathbf{0.017}$ & $\mathbf{3.5{\times}10^{-3}}$\\
balanced RF & $0.303$ & $0.415$ & $1.0$ & $0.690$ & $1.0$ & $1.0$\\
RUSBoost & $\mathbf{9.6{\times}10^{-4}}$ & $\mathbf{2.6{\times}10^{-3}}$ & $\mathbf{6.2{\times}10^{-3}}$ & $\mathbf{1.7{\times}10^{-3}}$ & $\mathbf{0.013}$ & $\mathbf{7.9{\times}10^{-3}}$\\
R\&R-resample & $\mathbf{4.8{\times}10^{-6}}$ & $\mathbf{1.6{\times}10^{-3}}$ & $\mathbf{2.3{\times}10^{-3}}$ & $\mathbf{4.1{\times}10^{-5}}$ & $\mathbf{3.0{\times}10^{-4}}$ & $\mathbf{3.5{\times}10^{-3}}$\\
HistGB & $0.303$ & $0.181$ & $0.655$ & $0.690$ & $1.0$ & $0.830$\\
\bottomrule
\end{tabular}
\end{table}

\begingroup\widowpenalty=10000
Because $p$-values alone do not show practical magnitude, Table~\ref{tab:effects}
reports paired effects for four strong comparators. Differences are formed
as DRR minus comparator at the same metric-specific tuned cut. The R\&R
resampler has intervals separated from zero on all six metrics. In contrast, the
intervals against cost-sensitive learning, balanced RF, and HistGB include zero
on the primary AP metric, consistent with the nonsignificant Holm comparisons.
\par\endgroup

\begin{table}[t]
\centering
\normalsize
\caption{Dataset-level paired effects for DRR$(D=128)$ against four strong
comparators. W/L is the number of DRR wins/losses (no ties); mean is the paired
mean difference and HL is the Hodges--Lehmann pseudomedian with a 95\% percentile
bootstrap interval from 10,000 resamples of datasets. Higher is better for every
metric. All comparisons have $N=24$. The lower AUC limit against HistGB is
$-4\times10^{-5}$.}
\label{tab:effects}
\setlength{\tabcolsep}{3pt}
\renewcommand{\arraystretch}{1}
\begin{tabular}{@{}llrrl@{}}
\toprule
metric & comparator & W/L & mean & HL [95\% CI]\\
\midrule
AP & cost-sensitive & 15/9 & $+0.016$ & $+0.016$ [$-0.001$,$+0.032$]\\
 & balanced RF & 14/10 & $+0.031$ & $+0.013$ [$-0.003$,$+0.056$]\\
 & R\&R-resample & 22/2 & $+0.092$ & $+0.090$ [$+0.042$,$+0.125$]\\
 & HistGB & 14/10 & $+0.029$ & $+0.015$ [$-0.008$,$+0.051$]\\
\addlinespace
AUC & cost-sensitive & 11/13 & $+0.003$ & $+0.002$ [$-0.005$,$+0.010$]\\
 & balanced RF & 9/15 & $-0.006$ & $-0.006$ [$-0.014$,$+0.004$]\\
 & R\&R-resample & 19/5 & $+0.018$ & $+0.016$ [$+0.008$,$+0.026$]\\
 & HistGB & 15/9 & $+0.015$ & $+0.010$ [$-0.000$,$+0.030$]\\
\addlinespace
bal.acc & cost-sensitive & 9/15 & $+0.002$ & $+0.000$ [$-0.006$,$+0.008$]\\
 & balanced RF & 12/12 & $+0.001$ & $0.000$ [$-0.009$,$+0.011$]\\
 & R\&R-resample & 18/6 & $+0.021$ & $+0.019$ [$+0.009$,$+0.030$]\\
 & HistGB & 13/11 & $+0.011$ & $+0.009$ [$-0.004$,$+0.030$]\\
\addlinespace
$F_1$ & cost-sensitive & 14/10 & $+0.011$ & $+0.008$ [$-0.003$,$+0.022$]\\
 & balanced RF & 13/11 & $+0.027$ & $+0.008$ [$-0.007$,$+0.042$]\\
 & R\&R-resample & 21/3 & $+0.071$ & $+0.068$ [$+0.030$,$+0.099$]\\
 & HistGB & 12/12 & $+0.017$ & $+0.008$ [$-0.009$,$+0.042$]\\
\addlinespace
MCC & cost-sensitive & 11/13 & $+0.004$ & $+0.000$ [$-0.009$,$+0.013$]\\
 & balanced RF & 12/12 & $+0.011$ & $+0.003$ [$-0.014$,$+0.023$]\\
 & R\&R-resample & 19/5 & $+0.058$ & $+0.043$ [$+0.018$,$+0.086$]\\
 & HistGB & 14/10 & $+0.019$ & $+0.008$ [$-0.008$,$+0.050$]\\
\addlinespace
G-mean & cost-sensitive & 11/13 & $+0.004$ & $+0.001$ [$-0.007$,$+0.013$]\\
 & balanced RF & 12/12 & $+0.003$ & $+0.000$ [$-0.011$,$+0.014$]\\
 & R\&R-resample & 18/6 & $+0.026$ & $+0.023$ [$+0.010$,$+0.037$]\\
 & HistGB & 13/11 & $+0.019$ & $+0.009$ [$-0.005$,$+0.039$]\\
\bottomrule
\end{tabular}
\end{table}

\section{Operating-point metrics under three cut policies}\label{app:cuts}
\begin{figure}[!htbp]
\centering
\includegraphics[width=0.98\linewidth]{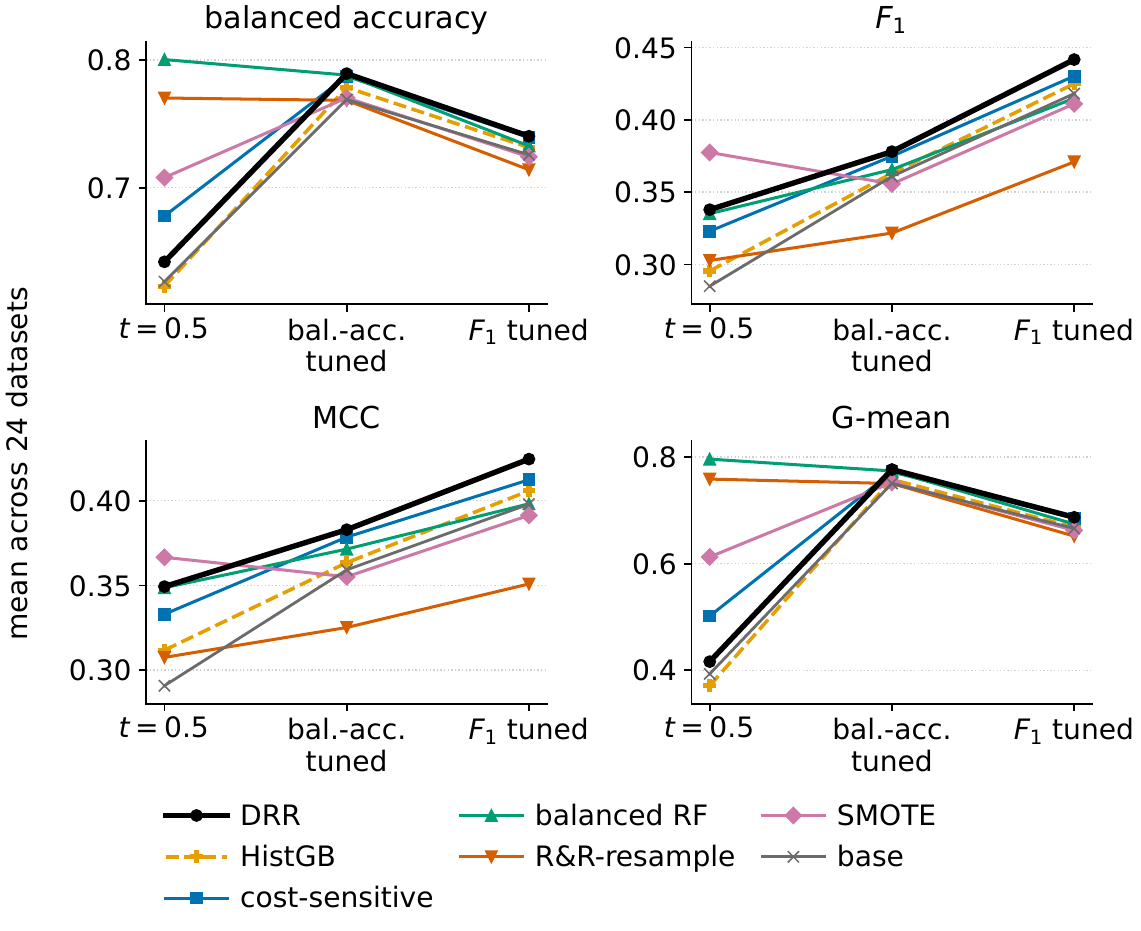}
\caption{Operating-point metrics of seven methods under the fixed $t=0.5$,
balanced-accuracy-tuned, and $F_1$-tuned cut policies (means over the 24
tabular benchmarks; every plotted method has complete results). DRR is
DRR$(D{=}128)$. Curves summarize each pipeline's test predictions at cuts
selected under the stated policy.}
\label{fig:cuts}
\end{figure}
The threshold-dependent metrics are reported here under each of the three cut
policies summarized in Figure~\ref{fig:cuts}: the fixed $t=0.5$
(Table~\ref{tab:cut05}), the balanced-accuracy-tuned cut
(Table~\ref{tab:cutbal}), and the $F_1$-tuned cut (Table~\ref{tab:cutf1}). Each
entry is an available-case mean: $N=22$ for SMOTE-ENN
and $N=24$ otherwise; best per column in bold. The ordering changes with the
cut policy: at $t=0.5$ several rebalancing methods and ensembles lead, while
DRR and the base trail. Under
either tuned cut DRR leads every column. Logit adjustment has no $t=0.5$ row
because its native operating point is the zero cut on the prior-shifted logit;
its tuned rows are rank-equivalent to those of \texttt{none}.

\begin{table}[!htbp]
\centering
\caption{Operating-point metrics at the fixed cut $t=0.5$, available-case means
($N=22$ for SMOTE-ENN and $N=24$ otherwise). Best per
column in bold.}
\label{tab:cut05}
\begin{tabular}{@{}lcccc@{}}
\toprule
Method & bal.acc & $F_1$ & MCC & G-mean\\
\midrule
\texttt{none} & 0.627 & 0.285 & 0.291 & 0.393\\
cost-sensitive & 0.678 & 0.323 & 0.333 & 0.502\\
logit adjust & -- & -- & -- & --\\
\texttt{SMOTE} & 0.708 & 0.377 & 0.367 & 0.613\\
\texttt{SMOTE-ENN} & 0.747 & \textbf{0.386} & \textbf{0.373} & 0.707\\
\texttt{FW-SMOTE} & 0.676 & 0.374 & 0.365 & 0.532\\
balanced RF & \textbf{0.800} & 0.335 & 0.349 & \textbf{0.796}\\
RUSBoost & 0.749 & 0.334 & 0.324 & 0.730\\
R\&R-resample$(D{=}128)$ & 0.770 & 0.303 & 0.307 & 0.759\\
HistGB & 0.623 & 0.295 & 0.312 & 0.371\\
\textbf{DRR}$(D{=}0)$ & 0.639 & 0.333 & 0.345 & 0.411\\
\textbf{DRR}$(D{=}128)$ & 0.642 & 0.338 & 0.349 & 0.416\\
\bottomrule
\end{tabular}
\end{table}

\begin{table}[!htbp]
\centering
\caption{Operating-point metrics at the balanced-accuracy-tuned cut, available-case
means ($N=22$ for SMOTE-ENN and $N=24$ otherwise). Best per
column in bold.}
\label{tab:cutbal}
\begin{tabular}{@{}lcccc@{}}
\toprule
Method & bal.acc & $F_1$ & MCC & G-mean\\
\midrule
\texttt{none} & 0.769 & 0.361 & 0.359 & 0.751\\
cost-sensitive & 0.787 & 0.375 & 0.379 & 0.772\\
logit adjust & 0.768 & 0.362 & 0.360 & 0.750\\
\texttt{SMOTE} & 0.770 & 0.356 & 0.355 & 0.752\\
\texttt{SMOTE-ENN} & 0.774 & 0.348 & 0.351 & 0.757\\
\texttt{FW-SMOTE} & 0.774 & 0.364 & 0.364 & 0.757\\
balanced RF & 0.788 & 0.366 & 0.371 & 0.774\\
RUSBoost & 0.757 & 0.329 & 0.326 & 0.742\\
R\&R-resample$(D{=}128)$ & 0.768 & 0.322 & 0.325 & 0.750\\
HistGB & 0.778 & 0.363 & 0.363 & 0.758\\
\textbf{DRR}$(D{=}0)$ & 0.786 & 0.375 & 0.379 & 0.773\\
\textbf{DRR}$(D{=}128)$ & \textbf{0.789} & \textbf{0.378} & \textbf{0.383} & \textbf{0.777}\\
\bottomrule
\end{tabular}
\end{table}

\begin{table}[!htbp]
\centering
\caption{Operating-point metrics at the $F_1$-tuned cut, available-case means
($N=22$ for SMOTE-ENN and $N=24$ otherwise). Best per
column in bold.}
\label{tab:cutf1}
\begin{tabular}{@{}lcccc@{}}
\toprule
Method & bal.acc & $F_1$ & MCC & G-mean\\
\midrule
\texttt{none} & 0.726 & 0.418 & 0.398 & 0.666\\
cost-sensitive & 0.739 & 0.430 & 0.412 & 0.685\\
logit adjust & 0.726 & 0.418 & 0.398 & 0.666\\
\texttt{SMOTE} & 0.724 & 0.411 & 0.391 & 0.663\\
\texttt{SMOTE-ENN} & 0.724 & 0.404 & 0.385 & 0.664\\
\texttt{FW-SMOTE} & 0.729 & 0.422 & 0.403 & 0.670\\
balanced RF & 0.733 & 0.415 & 0.398 & 0.674\\
RUSBoost & 0.713 & 0.371 & 0.349 & 0.655\\
R\&R-resample$(D{=}128)$ & 0.714 & 0.371 & 0.351 & 0.652\\
HistGB & 0.731 & 0.425 & 0.406 & 0.668\\
\textbf{DRR}$(D{=}0)$ & 0.737 & 0.438 & 0.421 & 0.683\\
\textbf{DRR}$(D{=}128)$ & \textbf{0.740} & \textbf{0.442} & \textbf{0.425} & \textbf{0.687}\\
\bottomrule
\end{tabular}
\end{table}

\section{Per-classifier results}\label{app:percls}
Table~\ref{tab:percls} reports average precision under each of the five base
classifiers separately, averaged over the $24$ tabular benchmarks;
DRR$(D{=}128)$ leads every column.

\begin{table}[!htbp]
\centering
\caption{Average precision by base classifier, averaged over the $24$ tabular
benchmarks. DRR$(D{=}128)$ has higher mean AP than the base and the R\&R
resampler under all five base classifiers, and than cost-sensitive learning
under the three applicable classifiers. Cost-sensitive learning was not run for
naive Bayes, whose implementation exposes no class weighting, nor for the MLP;
dashes mark these cells.}
\label{tab:percls}
\begin{tabular}{@{}lccccc@{}}
\toprule
base classifier & SVM & RF & NB & KLR & MLP\\
\midrule
\texttt{none} & 0.458 & 0.445 & 0.339 & 0.412 & 0.457\\
cost-sensitive & 0.432 & 0.458 & -- & 0.428 & --\\
R\&R-resample & 0.382 & 0.361 & 0.302 & 0.390 & 0.384\\
\textbf{DRR}$(D{=}128)$ & \textbf{0.471} & \textbf{0.470} & \textbf{0.423} & \textbf{0.437} & \textbf{0.476}\\
\bottomrule
\end{tabular}
\end{table}

\section{The fusion deflection}\label{app:theorycurve}
Figure~\ref{fig:theory} plots the fusion deflection $d'(w)$ as the fusion weight
$w$ varies, for the three regimes distinguished in the main manuscript. With
$r=d_2/d_1$ and $q=\pi_0\pi_1 d_1^2$, the exact population-marginal condition for
the deployed weight to improve separation, $d'(w_{\mathrm{eff}})>d_1$, is
$\tfrac12\sqrt{(1+q)/(1+qr^2)}\,(1-r^2)<2(r-\rho)$; it reduces to
$r>\sqrt{5+4\rho}-2$ as $q\to0$.

\begin{figure}[!htb]
\centering
\includegraphics[width=0.98\linewidth]{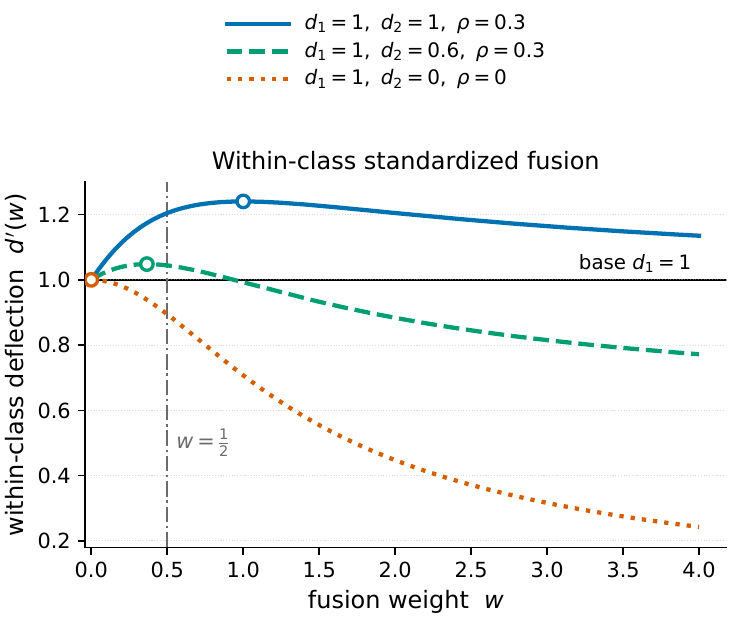}
\caption{Fusion deflection $d'(w)$ for within-class-standardized scores.
The three curves use $(d_1,d_2,\rho)=(1,1,0.3)$, $(1,0.6,0.3)$, and $(1,0,0)$.
An informative second score can improve separation, whereas a
zero-separation, uncorrelated score reduces it. Circles mark the optimal weights
under the stated model; the vertical line marks $w=1/2$ for reference.}
\label{fig:theory}
\end{figure}

\section{Same-feature logistic controls}\label{app:samephi}
This experiment asks whether the raking dual supplies more useful rescoring
information than an ordinary discriminatively trained score on the same
feature map. It uses the 24 tabular datasets, 30 stratified splits, and
five base learners of the main evaluation. All arms share freshly generated
base and dual predictions within each split. The fit, threshold, and test
indices follow the original protocol; the reconstructed feature arrays for
the logistic controls match the raking arrays exactly, verified by their
SHA256 hashes on all three subsets. The current drivers and imported-source
hashes are archived with the score caches; this matched comparison is
reported separately from the historical runs. An audit against the August
archive found absolute DRR panel-mean differences below $0.000018$ and
a largest absolute dataset-mean AP difference of $0.000495$; all comparisons
here use the same current frozen implementation within each split.
The experiment reports raw AP and AUC.

The primary control is $\ell_2$-penalized logistic regression (LR), with
$C=1$, an intercept, the L-BFGS solver, tolerance $10^{-6}$, and at most
2,000 iterations. It is fitted to the same $\phi_k$ coordinates as the dual,
without another feature scaler, on the original-prior fit subset.
Its decision function $\ell(x)$ supplies the second score. For each base
learner the matched fusion is
\[
 s_{\mathrm{LR}}(x)=z_M\!\big(f_{\mathrm{base}}(x)\big)
                  +\tfrac12 z_M\!\big(\ell(x)\big),
\]
using the same threshold-subset standardization and fixed coefficient as
DRR. A class-weighted LR with \texttt{class\_weight=balanced} is a secondary
control. All choices were fixed before the new results were inspected.
The secondary grid $C\in\{0.1,1,10,100\}$ is reported in full; no value is
selected from test performance. Standalone LR uses its raw decision
function, the dual uses its raw linear score, and the standalone base uses
its native probability. Within a fixed cache, floating-point ties can make the AP of a raw
score differ slightly from that of its positive affine standardization;
every fused arm shares exactly the same standardized base component.

Table~\ref{tab:samephi_primary} reports the $C=1$ comparison. We first
average over learners within a trial and then over trials within a dataset;
the 24 dataset values are the paired inferential units. The two primary
tests compare DRR with base$+$unweighted LR in AP at $D=0$ and $D=128$
using two-sided Wilcoxon signed-rank tests, with Holm correction across
these two tests. The intervals are percentile 95\% intervals for the
mean difference from 9,999 dataset-bootstrap draws. AUC, balanced LR,
and the other $C$ settings are secondary analyses.

At $C=1$, the mean AP differences DRR minus base$+$LR are $+0.0127$
at $D=0$ and $+0.0103$ at $D=128$.
The advantage depends on LR regularization
(Table~\ref{tab:samephi_grid}). At $D=128$, base$+$LR reaches AP
$0.4558$ at $C=10$ and
$0.4569$ at $C=100$, matching or slightly exceeding
DRR's $0.4555$. Standalone LR at these settings reaches
$0.4582$ and $0.4590$.
Thus, this comparison does not establish that raking is necessary for the
rescoring gain: a generic discriminative score on the same representation
can achieve similar performance. Matching the representation controls the
feature map; the learners still use different objectives and regularizers.

\begin{table}[t]
\centering
\normalsize
\setlength{\tabcolsep}{3pt}
\caption{Same-feature controls at the fixed primary value $C=1$.
Top: dataset-averaged raw AP and AUC. Bottom: DRR-minus-base$+$LR
AP differences, dataset-bootstrap intervals, wins/losses/ties, and
Holm-adjusted $p$-values for the two primary comparisons.}
\label{tab:samephi_primary}
\begin{tabular}{@{}lrrrr@{}}
\toprule
& \multicolumn{2}{c}{$D=0$} & \multicolumn{2}{c}{$D=128$}\\
\cmidrule(lr){2-3}\cmidrule(lr){4-5}
Method & AP & AUC & AP & AUC\\
\midrule
Base (native probability) & $0.4223$ & $0.8325$ & $0.4223$ & $0.8325$\\
Dual alone & $0.4116$ & $0.8394$ & $0.4500$ & $0.8564$\\
LR alone & $0.3771$ & $0.8293$ & $0.4099$ & $0.8503$\\
Balanced LR alone & $0.3943$ & $0.8500$ & $0.4328$ & $0.8654$\\
Base $+$ LR & $0.4373$ & $0.8534$ & $0.4452$ & $0.8572$\\
Base $+$ balanced LR & $0.4419$ & $0.8591$ & $0.4505$ & $0.8628$\\
\textbf{DRR} & $0.4501$ & $0.8559$ & $0.4555$ & $0.8581$\\
\bottomrule
\end{tabular}
\par\medskip
\begin{tabular}{@{}rrrrr@{}}
\toprule
$D$ & $\Delta$AP & 95\% CI & W/L/T & Holm $p$\\
\midrule
0 & $+0.0127$ & $[+0.0028,+0.0244]$ & 16/8/0 & $0.056$\\
128 & $+0.0103$ & $[+0.0033,+0.0178]$ & 17/7/0 & $0.042$\\
\bottomrule
\end{tabular}
\end{table}

\begin{table}[t]
\centering
\normalsize
\setlength{\tabcolsep}{3pt}
\caption{Full fixed-$C$ sensitivity for same-feature logistic controls.
Unweighted and balanced LR are each shown alone and fused with the
base at weight $1/2$. Every entry averages the same 24 datasets and
30 trials; no test-selected $C$ is used in the primary comparison.}
\label{tab:samephi_grid}
\begin{tabular}{@{}rrrrrr@{}}
\toprule
\multicolumn{6}{c}{AP}\\
\midrule
$D$ & $C$ & \shortstack{LR\\alone} & \shortstack{Base\\$+$ LR}
& \shortstack{Balanced\\LR alone} & \shortstack{Base $+$\\balanced LR}\\
\midrule
0 & 0.1 & $0.3540$ & $0.4323$ & $0.3645$ & $0.4347$\\
0 & 1 & $0.3771$ & $0.4373$ & $0.3943$ & $0.4419$\\
0 & 10 & $0.4136$ & $0.4464$ & $0.4233$ & $0.4495$\\
0 & 100 & $0.4393$ & $0.4523$ & $0.4317$ & $0.4515$\\
\midrule
128 & 0.1 & $0.3697$ & $0.4367$ & $0.3881$ & $0.4410$\\
128 & 1 & $0.4099$ & $0.4452$ & $0.4328$ & $0.4505$\\
128 & 10 & $0.4582$ & $0.4558$ & $0.4530$ & $0.4555$\\
128 & 100 & $0.4590$ & $0.4569$ & $0.4470$ & $0.4549$\\
\bottomrule
\end{tabular}
\par\medskip
\begin{tabular}{@{}rrrrrr@{}}
\toprule
\multicolumn{6}{c}{AUC}\\
\midrule
$D$ & $C$ & \shortstack{LR\\alone} & \shortstack{Base\\$+$ LR}
& \shortstack{Balanced\\LR alone} & \shortstack{Base $+$\\balanced LR}\\
\midrule
0 & 0.1 & $0.8205$ & $0.8514$ & $0.8278$ & $0.8533$\\
0 & 1 & $0.8293$ & $0.8534$ & $0.8500$ & $0.8591$\\
0 & 10 & $0.8531$ & $0.8593$ & $0.8629$ & $0.8623$\\
0 & 100 & $0.8606$ & $0.8605$ & $0.8596$ & $0.8606$\\
\midrule
128 & 0.1 & $0.8359$ & $0.8535$ & $0.8466$ & $0.8567$\\
128 & 1 & $0.8503$ & $0.8572$ & $0.8654$ & $0.8628$\\
128 & 10 & $0.8660$ & $0.8622$ & $0.8661$ & $0.8630$\\
128 & 100 & $0.8626$ & $0.8610$ & $0.8553$ & $0.8585$\\
\bottomrule
\end{tabular}
\end{table}

\section{Sensitivity to $D$ and $\eta$, separation estimates, and runtime}\label{app:sensitivity}
This appendix reports a supplementary run of the DRR arm alone on the 24
tabular benchmarks under the protocol of \S\ref{sec:proto}\ --- the same outer and inner splits,
encoders, feature-map normalization, tolerance rule, base learners, and marginal
standardization --- varying one ingredient at a time: the random-feature
dimension $D\in\{0,32,64,128,256,512\}$ at $\eta=0.05$; the tolerance fraction
$\eta\in\{0.01,0.05,0.2,0.5\}$ at $D=0$ and $D=128$; and the random-feature
draw (three seeds) at $D=128$. Every configuration yields one raw-score AP per
dataset, trial, and base learner; the tables report dataset means over $30$
trials and five learners, with 95\% percentile dataset-bootstrap intervals of
the paired mean difference from $9{,}999$ resamples. The reported configuration
reproduces the values of \S\ref{sec:results}\ exactly; the intervals use a fresh
resampling seed and can differ from those quoted there in the last digit. For
the five datasets without continuous variables the feature map has no
random-feature block, so their $D>0$ rows coincide with $D=0$.

Table~\ref{tab:sweepD} and Figure~\ref{fig:sensitivity} (left) show that the
gain over the base grows slowly with $D$, from $0.0285$ at $D=0$ to $0.0361$ at
$D=512$, and that DRR beats the base on 23 or 24 of the 24 datasets at every
$D$. Relative to the reported $D=128$, $D=64$ is within $0.002$, and $D=256$ and
$D=512$ are ahead by $0.002$ (intervals $[0.0003,0.0033]$ and
$[0.0001,0.0045]$), so $D=128$ is close to, but not at, the best setting on this
grid. The dual score alone accounts for most of this dependence, rising from
$0.412$ to $0.458$ AP, whereas the fused score moves by $0.008$.

The tolerance fraction matters little between $0.01$ and $0.2$ (fused AP within
$0.002$ of the reported $\eta=0.05$ at both resolutions) and degrades at
$\eta=0.5$, which places the tolerance halfway to the base discrepancy and
raises $\ESS(\what)/n_0$ from $0.12$ to $0.55$ at $D=128$: the fused gain falls
to $0.0239$ at $D=128$ and $0.0203$ at $D=0$ (Table~\ref{tab:sweepEta} and
Figure~\ref{fig:sensitivity}, right). At $D=0$ the tolerance reaches $\gamma_n$
in 12 of the 30 \texttt{yeast} trials, where the dual is zero
(\ref{app:repro}); no other dataset triggers this. Across three random-feature
draws at $D=128$ the panel mean AP is $0.455$, $0.457$, and $0.457$; the
per-dataset range across draws averages $0.005$ (maximum $0.022$, on
\texttt{arrhythmia}) for the fused score and $0.016$ (maximum $0.076$, on
\texttt{glass4}) for the dual score alone, so fusion also damps the
draw-to-draw variability of the dual.

\begin{table}[t]
\centering
\footnotesize
\setlength{\tabcolsep}{4pt}
\caption{Sensitivity of DRR to the random-feature dimension $D$ at $\eta=0.05$
on the 24 tabular benchmarks (base AP $0.422$ in every row): dataset means of
raw-score AP over 30 trials and five base learners, the paired gain over the
base with its 95\% dataset-bootstrap interval and win/loss count over datasets,
and the paired difference from the reported $D=128$.}
\label{tab:sweepD}
\begin{tabular}{@{}lccccc@{}}
\toprule
$D$ & AP dual & AP DRR & DRR$-$base [95\% CI] & W/L & vs.\ $D{=}128$ [95\% CI]\\
\midrule
0 & 0.412 & 0.450 & $+0.0285$ [$+0.0187$, $+0.0393$] & 23/1 & $-0.0054$ [$-0.0112$, $-0.0003$]\\
32 & 0.427 & 0.451 & $+0.0293$ [$+0.0197$, $+0.0401$] & 24/0 & $-0.0046$ [$-0.0091$, $-0.0010$]\\
64 & 0.441 & 0.454 & $+0.0323$ [$+0.0222$, $+0.0439$] & 23/1 & $-0.0016$ [$-0.0041$, $+0.0006$]\\
128 & 0.450 & 0.455 & $+0.0339$ [$+0.0232$, $+0.0460$] & 24/0 & --\\
256 & 0.455 & 0.457 & $+0.0357$ [$+0.0249$, $+0.0476$] & 24/0 & $+0.0018$ [$+0.0003$, $+0.0033$]\\
512 & 0.458 & 0.458 & $+0.0361$ [$+0.0251$, $+0.0485$] & 24/0 & $+0.0022$ [$+0.0001$, $+0.0045$]\\
\bottomrule
\end{tabular}
\end{table}

\begin{table}[t]
\centering
\footnotesize
\setlength{\tabcolsep}{4pt}
\caption{Sensitivity of DRR to the tolerance fraction $\eta$ of the rule
$\delta_n=\max\{\varphi_n+\eta(\gamma_n-\varphi_n),1.05\varphi_n\}$ at $D=0$
and $D=128$; columns as in Table~\ref{tab:sweepD}, with the last column
relative to the reported $(D,\eta)=(128,0.05)$.}
\label{tab:sweepEta}
\begin{tabular}{@{}lccccc@{}}
\toprule
$\eta$ & AP dual & AP DRR & DRR$-$base [95\% CI] & W/L & vs.\ $(128,0.05)$ [95\% CI]\\
\midrule
\multicolumn{6}{@{}l}{$D=0$}\\
0.01 & 0.406 & 0.449 & $+0.0273$ [$+0.0172$, $+0.0385$] & 22/2 & $-0.0066$ [$-0.0123$, $-0.0011$]\\
0.05 & 0.412 & 0.450 & $+0.0285$ [$+0.0187$, $+0.0393$] & 23/1 & $-0.0054$ [$-0.0112$, $-0.0003$]\\
0.2 & 0.406 & 0.448 & $+0.0266$ [$+0.0181$, $+0.0363$] & 23/1 & $-0.0073$ [$-0.0144$, $-0.0012$]\\
0.5 & 0.382 & 0.442 & $+0.0203$ [$+0.0133$, $+0.0278$] & 21/3 & $-0.0135$ [$-0.0232$, $-0.0052$]\\
\addlinespace
\multicolumn{6}{@{}l}{$D=128$}\\
0.01 & 0.440 & 0.454 & $+0.0323$ [$+0.0214$, $+0.0449$] & 23/1 & $-0.0015$ [$-0.0030$, $+0.0000$]\\
0.05 & 0.450 & 0.455 & $+0.0339$ [$+0.0232$, $+0.0460$] & 24/0 & --\\
0.2 & 0.446 & 0.454 & $+0.0322$ [$+0.0223$, $+0.0432$] & 23/1 & $-0.0017$ [$-0.0045$, $+0.0010$]\\
0.5 & 0.411 & 0.445 & $+0.0239$ [$+0.0158$, $+0.0327$] & 22/2 & $-0.0100$ [$-0.0173$, $-0.0034$]\\
\bottomrule
\end{tabular}
\end{table}

\begin{figure}[t]
\centering
\includegraphics[width=0.98\linewidth]{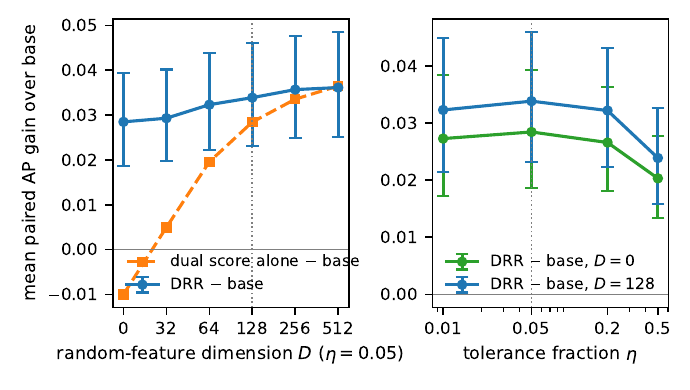}
\caption{Sensitivity of the mean paired raw-score AP gain over the base on the
24 tabular benchmarks. Left: random-feature dimension $D$ at $\eta=0.05$, for
DRR and for the dual score alone. Right: tolerance fraction $\eta$ at $D=0$ and
$D=128$. Bars are 95\% percentile dataset-bootstrap intervals; the dotted lines
mark the reported $D=128$ and $\eta=0.05$.}
\label{fig:sensitivity}
\end{figure}

Table~\ref{tab:separation} estimates the quantities of  Theorem~\ref{thm:fusion}\ on the test
half at the reported configuration: the pooled within-class deflections $d_1$
and $d_2$ of the standardized base and dual scores, their pooled within-class
correlation $\rho$, and the implied change $d'(\tfrac12)-d_1$ of the fused
deflection, averaged over trials and base learners. The implied change is
evaluated at the within-class weight $\tfrac12$, the $\pi_1\to0$ limit of the
effective weight $w_{\mathrm{eff}}$ of  Remark~\ref{rem:w}; evaluating it at
$w_{\mathrm{eff}}$ instead (mean $0.55$ over the $120$ cells; above $0.81$ in
only five cells, all on \texttt{vowel0} or \texttt{isolet}, with a maximum of
$1.29$) leaves every aggregate below unchanged (sign agreement $80\%$, 86 cells with a
predicted gain, Spearman $+0.73$). The within-class
correlation is moderate on every dataset (median $0.51$, interquartile range
$0.41$--$0.61$ over the $120$ dataset--learner cells), and the dual's deflection
is comparable to the base's (median $d_2/d_1=0.94$), above the fixed-weight
threshold $\sqrt{5+4\rho}-2$ ($0.52$--$0.84$) on most datasets. In 86 of the 120
cells the ratio exceeds the threshold, so the theorem predicts a gain; the
observed raw-score AP gain is positive in 108 cells, and the sign of the implied
$\Delta d'$ agrees with that of the observed gain in 80\% of cells (78\% when
the moments are estimated on the threshold subset instead of the test half).
Across cells the implied $\Delta d'$ and the observed AP gain have Spearman
correlation $+0.73$ (Figure~\ref{fig:separation}; $+0.61$ with
threshold-subset moments and $+0.40$ over the 24 dataset means). The
disagreements are informative: on \texttt{vowel0}, \texttt{isolet},
\texttt{led}, and \texttt{car-good} the base already separates the classes
strongly ($d_1$ between $3.9$ and $7.1$) and the two-moment summary predicts a
loss, yet AP improves (by $0.105$ on \texttt{isolet}). A plausible
explanation, which we have not tested, is that the gain there sits in the upper
tail of the score, which AP weights heavily and an equal-variance bi-normal
summary does not capture.

\begin{figure}[t]
\centering
\includegraphics[width=0.75\linewidth]{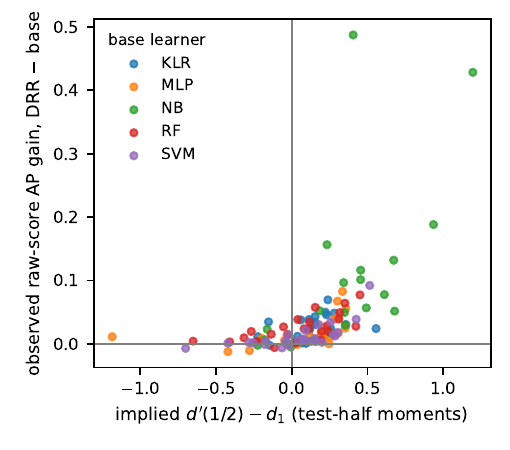}
\caption{Implied change in the fused deflection, $d'(\tfrac12)-d_1$, computed
from test-half within-class moments, against the observed raw-score AP gain of
DRR$(D{=}128)$ over the base, one point per dataset and base learner ($120$
points; Spearman $+0.73$).}
\label{fig:separation}
\end{figure}

\begin{table}[t]
\centering
\normalsize
\setlength{\tabcolsep}{3pt}
\caption{Per-dataset separation estimates for DRR$(D{=}128)$ on the test half,
averaged over 30 trials and the five base learners: within-class deflections
$d_1$ (base) and $d_2$ (dual score), within-class correlation $\rho$, the ratio
$d_2/d_1$ of the averaged deflections against the fixed-weight threshold
$\sqrt{5+4\rho}-2$, the implied change $\Delta d'=d'(\tfrac12)-d_1$, and the
observed raw-score AP gain of DRR over the base. The threshold and $\Delta d'$
columns average the per-learner values. Datasets are ordered by decreasing
$D=128$ feasibility floor, as in Table~\ref{tab:perdataset}.}
\label{tab:separation}
\begin{tabular}{@{}lrrrrrrr@{}}
\toprule
dataset & $d_1$ & $d_2$ & $\rho$ & $d_2/d_1$ & $\sqrt{5+4\rho}-2$ & $\Delta d'$ & $\Delta$AP\\
\midrule
\texttt{yeast} & 4.01 & 4.59 & 0.77 & 1.14 & 0.84 & $+0.41$ & $+0.024$\\
\texttt{glass4} & 3.41 & 3.17 & 0.55 & 0.93 & 0.68 & $+0.33$ & $+0.071$\\
\texttt{spectrometer} & 4.52 & 4.00 & 0.58 & 0.88 & 0.70 & $+0.30$ & $+0.111$\\
\texttt{libras\_move} & 4.28 & 3.19 & 0.46 & 0.75 & 0.61 & $+0.21$ & $+0.053$\\
\texttt{vowel0} & 7.14 & 3.59 & 0.52 & 0.50 & 0.66 & $-0.33$ & $+0.027$\\
\texttt{arrhythmia} & 0.90 & 0.87 & 0.35 & 0.97 & 0.52 & $+0.16$ & $+0.036$\\
\texttt{oil} & 2.57 & 2.33 & 0.56 & 0.91 & 0.69 & $+0.20$ & $+0.050$\\
\texttt{winequality} & 0.97 & 1.28 & 0.52 & 1.33 & 0.65 & $+0.23$ & $+0.014$\\
\texttt{hypo} & 5.57 & 4.10 & 0.57 & 0.74 & 0.70 & $+0.08$ & $+0.025$\\
\texttt{led} & 3.91 & 2.28 & 0.52 & 0.58 & 0.66 & $-0.10$ & $+0.011$\\
\texttt{mammography} & 5.56 & 3.86 & 0.53 & 0.69 & 0.67 & $+0.07$ & $+0.035$\\
\texttt{flare-f} & 1.21 & 1.69 & 0.52 & 1.39 & 0.66 & $+0.32$ & $+0.035$\\
\texttt{us\_crime} & 2.18 & 2.06 & 0.65 & 0.95 & 0.76 & $+0.15$ & $+0.056$\\
\texttt{isolet} & 5.46 & 2.49 & 0.43 & 0.46 & 0.59 & $-0.23$ & $+0.105$\\
\texttt{abalone} & 2.16 & 2.02 & 0.48 & 0.94 & 0.63 & $+0.25$ & $+0.040$\\
\texttt{scene} & 0.92 & 0.87 & 0.48 & 0.94 & 0.63 & $+0.11$ & $+0.021$\\
\texttt{yeast\_ml8} & 0.20 & 0.15 & 0.43 & 0.72 & 0.59 & $+0.01$ & $+0.000$\\
\texttt{boundary} & 0.70 & 0.94 & 0.46 & 1.34 & 0.61 & $+0.19$ & $+0.013$\\
\texttt{car-good} & 4.27 & 1.87 & 0.51 & 0.44 & 0.65 & $-0.28$ & $+0.001$\\
\texttt{phoss} & 0.95 & 0.80 & 0.51 & 0.84 & 0.65 & $+0.07$ & $+0.010$\\
\texttt{abalone\_19} & 0.53 & 0.80 & 0.48 & 1.53 & 0.62 & $+0.17$ & $+0.003$\\
\texttt{compustat} & 1.47 & 1.21 & 0.40 & 0.82 & 0.57 & $+0.14$ & $+0.012$\\
\texttt{sick} & 0.92 & 1.30 & 0.58 & 1.41 & 0.70 & $+0.23$ & $+0.032$\\
\texttt{cam} & 1.34 & 1.29 & 0.50 & 0.97 & 0.64 & $+0.16$ & $+0.029$\\
\bottomrule
\end{tabular}
\end{table}

On the laptop used for this run (an eight-core mobile processor with one
single-threaded worker per core), the dual solve --- feasibility floor plus
tilt --- took a median of $0.33$ s per fit subset at $D=0$ (all $720$ fit
subsets) and, over the $570$ fit subsets whose feature map carries a
random-feature block (the 19 datasets with at least one continuous variable),
$0.50$ s at $D=128$ and $0.70$ s at $D=512$, with maxima
of $191$ s (\texttt{phoss}, $D=0$; an all-categorical design whose map is the
same at every $D$), $55$ s (\texttt{isolet}, $D=128$), and $77$ s
(\texttt{compustat}, $D=512$). The dual solver --- damped Newton up to $600$
feature dimensions and L-BFGS-B above it, the path taken by \texttt{isolet}
and \texttt{phoss} at every $D$ and by six further datasets at $D=512$ ---
converged in every reported-configuration solve (median
$12$ iterations, maximum $83$) and in all but one of the $10{,}080$ solves
overall (an L-BFGS-B solve on \texttt{isolet} at $D=128$, $\eta=0.5$). For
comparison, fitting the base learners took medians of $0.05$ s (SVM), $0.5$ s
(random forest), and $1.0$ s (MLP) per fit subset, with maxima of $130$ s,
$41$ s, and $22$ s.

\end{document}